\documentclass{article}

\usepackage{iclr2027_conference,times}

\usepackage[utf8]{inputenc} 
\usepackage[T1]{fontenc}    
\usepackage[hyperfootnotes=false]{hyperref}       
\usepackage{url}            
\usepackage{booktabs}       
\usepackage{amsfonts}       
\usepackage{nicefrac}       
\usepackage{microtype}      
\usepackage{xcolor}         
\usepackage{caption}
\usepackage[most]{tcolorbox}

\usepackage{amsmath,amsfonts,bm}

\def\secref#1{section~\ref{#1}}

\def\eqref#1{equation~\ref{#1}}

\def\1{\bm{1}}

\DeclareMathAlphabet{\mathsfit}{\encodingdefault}{\sfdefault}{m}{sl}
\SetMathAlphabet{\mathsfit}{bold}{\encodingdefault}{\sfdefault}{bx}{n}

\newcommand{\E}{\mathbb{E}}

\newcommand{\R}{\mathbb{R}}

\newcommand{\Var}{\mathrm{Var}}

\newcommand{\Cov}{\mathrm{Cov}}

\DeclareMathOperator*{\argmax}{arg\,max}

\DeclareMathOperator{\sign}{sign}

\usepackage{amsmath,amssymb,amsthm,mathtools}
\usepackage{amsfonts,bm}
\usepackage{booktabs}
\usepackage{algorithm}
\usepackage{algpseudocode}
\usepackage{hyperref}
\usepackage{xurl} %
\usepackage{microtype}
\usepackage{enumitem}
\usepackage{xcolor}
\usepackage{graphicx}
\usepackage{subcaption}
\usepackage{multirow}
\usepackage{xspace}
\usepackage{nicefrac}

\usepackage{booktabs}
\usepackage{makecell} %

\definecolor{darkblue}{rgb}{0, 0, 0.5}
\definecolor{revisionblue}{RGB}{0,70,180}
\newif\ifshowupdates
\showupdatesfalse
\ifshowupdates
  \newcommand{\UPDATE}[1]{{\color{revisionblue}#1}}
\else
  \newcommand{\UPDATE}[1]{{\color{black}#1}}
\fi
\hypersetup{colorlinks=true, citecolor=darkblue, linkcolor=darkblue, urlcolor=darkblue}

\usepackage{tikz}
\usetikzlibrary{arrows.meta,decorations.pathreplacing,calc,positioning}

\newtheorem{theorem}{Theorem}
\newtheorem{proposition}[theorem]{Proposition}
\newtheorem{lemma}[theorem]{Lemma}
\newtheorem{corollary}[theorem]{Corollary}
\newtheorem{definition}{Definition}
\newtheorem{remark}{Remark}

\newcommand{\inner}[2]{\left\langle #1,\,#2\right\rangle}
\newcommand{\norm}[1]{\left\|#1\right\|}
\newcommand{\op}{\mathrm{op}}
\newcommand{\nuc}{*}
\newcommand{\F}{\mathrm{F}}

\newcommand{\polar}{\mathrm{polar}}

\newcommand{\specnorm}[1]{\|#1\|_{\mathrm{op}}}
\newcommand{\nucnorm}[1]{\|#1\|_{\mathrm{nuc}}}
\newcommand{\onenorm}[1]{\|#1\|_{1}}

\newcommand{\RLB}{\UPDATE{\mathcal{R}_{\mathrm{LB}}}}
\newcommand{\Rinft}{\mathcal{R}_{\infty}}
\newcommand{\Rzer}{\mathcal{R}_{0}}
\newcommand{\Rhat}{\UPDATE{\widehat{\mathcal{R}}_{\mathrm{dir}}}}
\newcommand{\Dspec}{D_{\mathrm{spec}}}
\newcommand{\Dsign}{D_{\mathrm{sign}}}
\newcommand{\Lop}{L_{\mathrm{op}}}
\newcommand{\Linf}{L_{\infty}}
\newcommand{\SpecSGD}{\textsc{SpecSGD}}
\newcommand{\SignSGD}{\textsc{SignSGD}}

\providecommand{\tr}{\mathrm{tr}}
\providecommand{\rank}{\mathrm{rank}}

\providecommand{\E}{\mathbb{E}}

\providecommand{\R}{\mathbb{R}}

\providecommand{\Var}{\mathrm{Var}}

\def\secref#1{Section~\ref{#1}}

\def\tabref#1{Table~\ref{#1}}

\makeatletter
\def\eqref#1{\textup{\tagform@{\ref{#1}}}}
\makeatother

\newcommand{\strank}[1]{\mathrm{st}(#1)}
\newcommand{\nrank}[1]{\mathrm{nr}(#1)}
\newcommand{\NSR}{\mathrm{NSR}}
\newcommand{\ssigma}{s_{\sigma}}
\newcommand{\effrank}[1]{\mathrm{er}(#1)}
\def\Propref#1{Proposition~\ref{#1}}

\newif\ifshowtodos
\showtodosfalse
\DeclareRobustCommand{\TODO}[1]{%
  \ifshowtodos
    \textcolor{red}{\textbf{[TODO:} #1]}%
  \fi
}
\hypersetup{pdfauthor={},pdftitle={Which Geometry on Which Layer? A Principled Criterion for Mixed-Optimizer Training}}

\iclrfinalcopy
\hypersetup{
  pdftitle={Which Geometry on Which Layer? Measured Sign--Spectral Assignment in Mixed-Optimizer Training},
  pdfauthor={Hiroki Naganuma; Laura Gomezjurado; Mahdi Ghaznavi; Ganesh Talluri; Atsushi Nitanda; Seng Pei Liew; Ryuichiro Hataya; Ioannis Mitliagkas}
}

\title{Which Geometry on Which Layer?\\Measured Sign--Spectral Assignment\\in Mixed-Optimizer Training}

\author{
  \small Hiroki Naganuma\thanks{Equal contribution.}\hspace{2mm}\textsuperscript{$\dagger$ 1,2} \quad
  Laura Gomezjurado\footnotemark[1]\hspace{2mm}\textsuperscript{3} \quad
  Mahdi Ghaznavi\textsuperscript{1,2} \\
  \small\bfseries Ganesh Talluri\textsuperscript{4} \quad
  Atsushi Nitanda\textsuperscript{5,6} \quad
  Seng Pei Liew\textsuperscript{7} \\
  \small\bfseries Ryuichiro Hataya\textsuperscript{7} \quad
  Ioannis Mitliagkas\textsuperscript{1,2}
  \\[0.5em]
  \normalfont\footnotesize \textsuperscript{1}Mila \quad
  \textsuperscript{2}Universit\'{e} de Montr\'{e}al \quad
  \textsuperscript{3}Stanford University \\
  \normalfont\footnotesize \textsuperscript{4}Arizona State University \quad
  \textsuperscript{5}A$^{*}$STAR \quad
  \textsuperscript{6}Nanyang Technological University \\
  \normalfont\footnotesize \textsuperscript{7}SB Intuitions Corp.
}

\begin{document}

\maketitle
\lhead{Preprint}

\ificlrfinal
\begingroup
\renewcommand{\thefootnote}{$\dagger$}
\footnotetext{Now at NVIDIA}
\endgroup
\setcounter{footnote}{0}
\fi

\begin{abstract}
Modern training of large neural networks now commonly uses mixed recipes, assigning AdamW-like or sign-based updates to embedding and output layers and spectral updates to hidden matrices. These assignments are fixed by layer type for the entire run and are largely inherited heuristics that must be re-tuned across architectures and training regimes, without a clear criterion for selecting the appropriate geometry, that is, the norm under which the steepest descent step is taken. We turn the sign-versus-spectral assignment into a measurable, layerwise decision problem. For ideal memoryless steps, we derive the finite-batch geometry ratio
 $\RLB(W;B):=B_{\mathrm{spec}}(W)/B_{\mathrm{sign}}(W),
$ where $B_{\mathrm{spec}}$ and $B_{\mathrm{sign}}$ are conservative descent-lemma lower bounds for the specified spectral and sign steps. Its closed form separates signal efficiency from a geometry-dependent global curvature factor. In the positive-signal regime, the same smoothness assumption gives two-sided bounds that determine which step has the larger conditional expected decrease when the ratio crosses explicit noise-dependent thresholds. We then derive a local directional quadratic score from gradient alignments and Hessian--vector products and use it in Generalized-Scion (G-Scion) to control layerwise routing.
In a controlled input-rank experimental setup, including directional curvature matches the measured same-iterate ordering in all seven aggregate settings, whereas the curvature-free proxy selects the opposite geometry. On $31$M GPT-2, the output score routes the output layer from spectral warmup to sign, and G-Scion improves validation loss over fixed Scion in memoryless and matched-momentum comparisons. ViT-B/16 results further show that a spectral edge assignment can outperform fixed Scion outside language modelling. These results suggest that optimizer geometry need not be hard-coded by layer type, and they are a first step toward understanding how it should be assigned.

\end{abstract}

\section{Introduction}
\label{sec:intro}
The training of large-scale neural networks has entered the era of mixed optimizers~\citep{Liu25b,zeng2025glm}.
State-of-the-art recipes no longer apply a single update rule to every weight matrix; instead, they partition the model by layer type, typically applying \UPDATE{AdamW or sign-based methods such as SignSGD}~\citep{kingma2015adam,bernstein2018signsgd} to embedding and output layers, while reserving spectral methods like Muon~\citep{jordan2024muon} for hidden weight matrices.

Current practice relies on inherited conventions, e.g., a recipe that works for a 7B language model is ported to a vision transformer or a multimodal architecture under the assumption that ``embedding-like'' and ``hidden-like'' roles remain invariant.
When these recipes fail, practitioners must resort to exhaustive hyperparameter sweeps to re-validate the mix.
This points to a gap in optimization theory: we lack a principled criterion that says which geometry belongs on which layer, how that preference evolves across training stages, and \UPDATE{whether architectural choices alter the preference between geometries}.

We argue that the choice between optimizer geometries is a measurable property of the local loss geometry whose usefulness must be tested against tuned static assignments. \UPDATE{Every recipe that mixes sign-based and spectral methods makes an implicit layerwise assignment. We analyse this decision through the finite-batch geometry ratio}
  \[
    \UPDATE{\RLB(W; B) \;:=\; B_{\mathrm{spec}}(W) \,/\, B_{\mathrm{sign}}(W),
}\]
where \UPDATE{$B_{\mathrm{spec}}$ and $B_{\mathrm{sign}}$ are conservative lower bounds for the conditional expected decreases of ideal memoryless spectral and sign steps. Thus $\RLB>1$ exactly orders these two lower-bound quantities. Proposition~\ref{prop:lower_bound_sandwich} strengthens the comparison: beyond explicit thresholds, the same ratio also guarantees the ordering of the corresponding conditional expected decreases. The ratio} factorizes into two independently interpretable terms:
\begin{itemize}
    \item \textbf{Signal efficiency}: A state-dependent term governed by the gradient's norms and minibatch noise.
    \item \textbf{\UPDATE{Curvature ratio}} ($\kappa$): A configuration-dependent term   \UPDATE{formed from geometry-specific global smoothness constants}.
\end{itemize}

\UPDATE{The factorisation yields testable predictions across layer type, training stage, model scale, and architecture. Here we evaluate the layer and stage predictions with the local directional score introduced in \secref{sec:method}; the same criterion defines matched tests across scale and architecture. The directional score operationalizes the signal--curvature view locally, but is not generally equal to the global-smoothness lower-bound ratio.}

\paragraph{Contributions.}
 \UPDATE{We summarize our contributions as follows:
} \begin{itemize}
    \item \textbf{\UPDATE{Analytical geometry comparison}}: We derive \UPDATE{a finite-batch ratio of geometry-specific lower bounds, factor it into signal efficiency and global curvature, and give explicit thresholds beyond which it orders the conditional expected decreases of the specified conservative steps} (\secref{sec:theory}).
    \item \textbf{\UPDATE{Measurement-driven assignment}}: We derive \UPDATE{an HVP-based local directional score and use it in G-Scion to replace Scion's fixed sign--spectral layer convention with online assignment} (\secref{sec:method}).
    \item \textbf{\UPDATE{Geometry diagnostics and training outcomes}}:
    \UPDATE{Controlled experiments demonstrate the value of directional-curvature information, while language-model checkpoint comparisons reveal stage-dependent geometry preferences. G-Scion improves over fixed Scion in the reported $31$M and ViT settings. Paired held-out update tests distinguish the score's predictive accuracy from the benefits of the resulting training policy}
    (\secref{sec:lm_validation}, \secref{sec:vlm_experiments}, and Appendix~\ref{app:scheduled_calibration}).
    \end{itemize}

\section{\UPDATE{Geometry Comparison: Global Bounds and Directional Scores}}
\label{sec:theory}

\UPDATE{We first derive geometry-specific lower bounds, compare them through the finite-batch ratio $\RLB$, and factor that ratio into signal and curvature terms. A two-sided result then relates the lower-bound comparison to actual conditional expected decreases for the specified conservative steps. The online quantity $\Rhat$, defined in \secref{sec:method}, is a local directional quadratic score rather than an estimator of $\RLB$; Appendix~\ref{app:notation} summarizes the notation.}

Let $W\in\R^{m\times n}$ be a weight matrix and $f:\R^{m\times n}\to\R$ a differentiable loss with gradient $G=\nabla f(W)$, thin SVD $G=U\Sigma V^\top$, and $p=\min(m,n)$. Let $\widehat{G}=G+E$ be a minibatch estimate with additive noise $E$ at batch size $B$.

\subsection{Norm-steepest descent}

 Given a primal norm $\norm{\cdot}$, the steepest-descent direction solves $D^*(G) \in \argmax_{\norm{D}\le 1} \inner{G}{D}$, with $\inner{A}{B}=\tr(A^\top B)$. By duality, $\max_{\norm{D}\le 1}\inner{G}{D}=\norm{G}_*$, the dual norm of $G$. Three geometries yield three familiar optimizers (\tabref{tab:geometries}).

\begin{table}[t]
\centering
\caption{Three optimizer geometries as norm-steepest descent.}
\label{tab:geometries}
\small
\begin{tabular}{lccc}
\toprule
\textbf{Method} & \textbf{Primal norm} $\norm{\cdot}$ & \textbf{Dual norm} $\norm{\cdot}_*$ & \textbf{Direction} $D^*(G)$ \\
\midrule
(S)GD & Frobenius $\norm{\cdot}_\F$ & Frobenius $\norm{\cdot}_\F$ & $G/\norm{G}_\F$ \\
SignSGD & Entrywise $\ell_\infty$ & Entrywise $\ell_1$ & $\sign(G)$ \\
SpecSGD & Spectral $\specnorm{\cdot}$ & Nuclear $\nucnorm{\cdot}$ & $\polar(G) = UV^\top$ \\
\bottomrule
\end{tabular}
\end{table}

If $f$ is $L_{\norm{\cdot}}$-smooth w.r.t.\ $\norm{\cdot}$, the ideal scaled step $W^+ = W - \tfrac{\norm{G}_*}{L_{\norm{\cdot}}}\,D^*(G)$ satisfies
\begin{equation}\label{eq:ideal_step}
    f(W^+) \;\le\; f(W) - \frac{\norm{G}_*^2}{2\,L_{\norm{\cdot}}}.
\end{equation}
  \UPDATE{For the least norm-specific smoothness constants, $L_\F\le\Lop\le pL_\F$ and $\Lop\le\Linf\le mnL_\F$ (Appendix~\ref{app:smoothness}); hence $1\le\kappa:=\Linf/\Lop\le mn$}.

\subsection{\UPDATE{Stochastic Expected-Decrease Bounds}}

  \UPDATE{Let $D_g=D_g^*(\widehat G)$ satisfy $\norm{D_g}_g\le1$ and define the conditional mean alignment $\mu_g:=\E[\inner{G}{D_g}\mid W]$. For any nonnegative scalar step $\eta$, norm-specific smoothness gives
}\begin{equation}\UPDATE{\label{eq:expected_gain}
    \E[f(W)-f(W-\eta D_g)\mid W]
    \;\ge\; \eta\mu_g-\frac{L_g\eta^2}{2}.
}\end{equation}
 \UPDATE{The best nonnegative step in this lower model is $\eta=(\mu_g)_+/L_g$, yielding the lower bound $(\mu_g)_+^2/(2L_g)$. } For $\Dsign=\sign(\widehat{G})$, the sign-bit alignment bound of \citet{balles2020geometry} gives   \UPDATE{$\mu_{\mathrm{sign}} \ge \onenorm{G} - 2\nu_1(W)$ } with \UPDATE{$\nu_1(W):= \E[\onenorm{\widehat G-G}\mid W]$}, hence   \UPDATE{the conservative step $\eta_{\mathrm{sign}}=(\onenorm{G}-2\nu_1(W))_+/\Linf$ gives
}\begin{equation}\UPDATE{\label{eq:sign_gain}
    B_{\mathrm{sign}}(W) \;:=\; \frac{\bigl(\onenorm{G} - 2\nu_1(W)\bigr)_+^2}{2\Linf}.
}\end{equation}
For $\Dspec=\polar(\widehat{G})$, the elementary inequality $\inner{G}{\polar(G+E)}\ge\nucnorm{G}-2\nucnorm{E}$ (Appendix~\ref{app:polar_proof}) with \UPDATE{$\nu_\nuc(W):= \E[\nucnorm{E}\mid W]$} gives
  \begin{equation}\UPDATE{\label{eq:spec_gain}
    B_{\mathrm{spec}}(W) \;:=\; \frac{\bigl(\nucnorm{G} - 2\nu_\nuc(W)\bigr)_+^2}{2\Lop},
}\end{equation}
 \UPDATE{with the analogous conservative spectral step $\eta_{\mathrm{spec}}=(\nucnorm{G}-2\nu_\nuc(W))_+/\Lop$. } Following \citet{naganuma2026adaptivebatchsizesusing}, write $c_1:= \sqrt{B}\,\nu_1$ and $c_\nuc:= \sqrt{B}\,\nu_\nuc$.   \UPDATE{Treating these quantities as approximately batch-independent is a modelling assumption, and neither norm alone implies a universal small-batch advantage: the relative penalties $c_1/\onenorm{G}$ and $c_\nuc/\nucnorm{G}$ determine the direction of the effect}.

\subsection{Definition and factorisation}

Taking the ratio of \UPDATE{the two lower-bound quantities in Eqs.~\eqref{eq:sign_gain}--\eqref{eq:spec_gain} gives the finite-batch geometry ratio used in the theoretical factorisation}.

  \begin{definition}[\UPDATE{Finite-batch geometry ratio}] \label{def:advantage_index}
For a layer with gradient $G$, batch size $B$, geometry-specific noise scales $c_1, c_\nuc$, and smoothness constants $\Linf, \Lop$,
  \begin{equation}\UPDATE{\label{eq:advantage_index}
    \RLB(W, B) \;:=\; \frac{B_{\mathrm{spec}}(W)}{B_{\mathrm{sign}}(W)}
    \;=\; \underbrace{\left(\frac{(\nucnorm{G} - 2c_\nuc/\sqrt{B})_+}{(\onenorm{G} - 2c_1/\sqrt{B})_+}\right)^{\!2}}_{\text{signal efficiency}}
    \,\cdot\,
    \underbrace{\frac{\Linf}{\Lop}}_{\text{curvature ratio }\kappa}.
}\end{equation}
 \UPDATE{This ratio is defined when $B_{\mathrm{sign}}>0$. If either lower bound is zero, the two nonnegative quantities should be compared directly rather than through a potentially undefined ratio.
} \end{definition}

$\RLB(W,B)>1$ \UPDATE{means that the spectral lower-bound quantity is larger than the sign lower-bound quantity at the same iterate}; $\RLB<1$ \UPDATE{means} the reverse. The signal-efficiency factor depends on the current gradient\UPDATE{, while the global curvature ratio depends on the objective and parameterization.
}

\begin{proposition}[\UPDATE{Two-sided expected-decrease bounds}]\label{prop:lower_bound_sandwich}
\UPDATE{Assume that $f$ is $L_g$-smooth with $0<L_g<\infty$ under each geometry $g\in\{\mathrm{sign},\mathrm{spec}\}$. Let $M_g=\norm{G}_{g,*}$, let $\nu_{\mathrm{sign}}=\nu_1(W)$ and $\nu_{\mathrm{spec}}=\nu_\nuc(W)$, and set}
\begin{equation}
\UPDATE{S_g=M_g-2\nu_g>0,\qquad \rho_g=\frac{2\nu_g}{M_g},\qquad \eta_g=\frac{S_g}{L_g},\qquad B_g=\frac{S_g^2}{2L_g}.}
\end{equation}
\UPDATE{For $D_g=D_g^*(\widehat G)$, define the conditional expected decrease at this conservative step by $\Delta_g:=\E[f(W)-f(W-\eta_gD_g)\mid W]$. Then}
\begin{equation}
\UPDATE{B_g\le \Delta_g\le c_gB_g,
\qquad c_g=\frac{3-\rho_g}{1-\rho_g}.}
\label{eq:decrease_sandwich}
\end{equation}
\UPDATE{Consequently, when both lower-bound quantities are positive,}
\begin{equation}
\UPDATE{\frac{\RLB}{c_{\mathrm{sign}}}
\le \frac{\Delta_{\mathrm{spec}}}{\Delta_{\mathrm{sign}}}
\le c_{\mathrm{spec}}\RLB.}
\label{eq:ratio_sandwich}
\end{equation}
\UPDATE{Thus $\RLB>c_{\mathrm{sign}}$ guarantees that the specified spectral step has the larger conditional expected decrease, whereas $\RLB<1/c_{\mathrm{spec}}$ gives the analogous guarantee for sign. At zero noise, these inequalities leave the ordering unresolved for $\RLB\in[1/3,3]$.}
\end{proposition}

The proof is given in Appendix~\ref{app:additional_theory}.

\begin{remark}[\UPDATE{From lower bounds to decrease ordering}]\label{rem:lower_bound_ratio}
\UPDATE{The ratio $B_{\mathrm{spec}}/B_{\mathrm{sign}}$ is not generally equal to either the conditional expected-decrease ratio or a realised sample ratio. Proposition~\ref{prop:lower_bound_sandwich} upgrades the lower-bound comparison into an ordering guarantee for the stated conservative steps beyond the two thresholds, while leaving the intermediate range unresolved. The local directional score in Eq.~\eqref{eq:empirical_index} is a separate construction used for online routing.}
\end{remark}

\subsection{\UPDATE{The Finite-Batch Ratio and Its Simplifications}}

The \UPDATE{finite-batch geometry ratio} $\RLB$ is \UPDATE{operationally heavy because estimating $c_1,c_\nuc,\Linf,\Lop$ on the fly is expensive. Two simplifications} drop terms under stated assumptions.

\paragraph{High-SNR limit $\Rinft$.}
When $\nucnorm{G}\gg 2 c_\nuc/\sqrt{B}$ and $\onenorm{G}\gg 2 c_1/\sqrt{B}$, drop the noise corrections:
\begin{equation}\label{eq:high_snr}
    \Rinft(W) \;:=\; \frac{\nucnorm{G}^2}{\onenorm{G}^2}\cdot\kappa
    \;=\; \frac{\mathrm{nr}(G)}{mn\cdot\phi_{\mathrm{ent}}(G)}\cdot\kappa,
\end{equation}
where $\mathrm{nr}(G) = \nucnorm{G}^2/\norm{G}_\F^2$ is the nuclear rank and $\phi_{\mathrm{ent}}(G)=\onenorm{G}^2/(mn\norm{G}_\F^2)$ is the entrywise density. \UPDATE{This simplification is valid only when both stated relative-noise terms are small.} $\Rinft$ is the \UPDATE{high-SNR form of the global-smoothness lower-bound ratio and motivates the local directional score}.

\paragraph{Curvature-free proxy $\Rzer$.}
Drop the curvature ratio as well:
\begin{equation}\label{eq:is_def}
    \Rzer(W) \;:=\; \frac{\nucnorm{G}^2}{\onenorm{G}^2}.
\end{equation}
The three forms therefore satisfy
\[
    \RLB(W;B) \;=\; \Rinft(W)\cdot [\text{noise correction factor}],
    \qquad
    \Rinft(W) \;=\; \Rzer(W)\cdot \kappa.
\]
For nonzero $G$, $\Rzer\le1$ because $\nucnorm{G}\le\onenorm{G}$; raw $\Rzer$ cannot favor spectral at a unit threshold. A dimension-normalized proxy $I_{\mathrm{proxy}}=\sqrt{mn\,\Rzer}$ is a different score: its unit-threshold decision is equivalent to substituting $\kappa=mn$, not to estimating curvature. Thus raw and normalized proxy readings must not be interchanged. $\Rzer$ approximates $\Rinft$ only when $\kappa\approx1$; the controlled sweep in Appendix~\ref{app:is_bounds} tests measured directional curvature, not the global $\kappa$.

\subsection{\UPDATE{Architecture-Dependent Curvature Hypothesis}}
\label{sec:activation_theory}

  \UPDATE{\citet[Sec.~7]{DavisDrusvyatskiy2025SpecGD} relate feature-based spectral and Frobenius majorizer coefficients through the stable rank of the incoming activations. Their analysis does not identify $L_\infty$ and therefore does not determine the sign--spectral curvature ratio $\kappa=L_\infty/L_\op$. We use activation stable rank as motivation for measuring architecture-dependent directional curvature; Appendix~\ref{app:activation_theory} gives the precise scope of the cited result.}

The activation moment ratio, its assumptions and the homogeneous specialization are collected in Appendix~\ref{app:additional_theory}; this motivation does not identify the sign--spectral smoothness ratio.

\subsection{Predictions}

\UPDATE{The factorisation predicts that the measured geometry score can depend on layer, training stage, scale, and architecture. The present directional-score evidence tests layer and stage dependence: the $31$M output lies on the sign side at the conservative learning rate, while a learning-rate-stratified aggregate changes sign across phases. Scale and architecture dependence remain falsifiable predictions for matched experiments that jointly verify the score, realised route, and static geometry controls.}

\section{\UPDATE{G-Scion: Measurement-Driven Geometry Assignment}}
\label{sec:method}

\UPDATE{The ratio $\RLB$ provides the analytical decomposition, but evaluating its global constants $\Linf$ and $\Lop$ online is expensive and conservative. G-Scion operationalizes the same signal--curvature view with a local directional quadratic score and uses that measurable quantity to route each eligible layer between sign and spectral updates.}

\begin{algorithm}[t]
\caption{G-Scion}
\label{alg:g-scion}
\begin{algorithmic}[1]
\State \textbf{Input:} cadence $N$, samples $K$, threshold $\tau$
\Statex \hspace{\algorithmicindent}\UPDATE{margin $m$, EMA factor $\beta$, confirmations $q$, warmup $t_{\mathrm w}$, eligible set $L_{\text{elig}}$, initial $\{g_\ell\}$}
\State $L_{\text{lock}} \gets \emptyset$\UPDATE{; $h_\ell\gets0$ for $\ell\in L_{\text{elig}}$}
\For{training step $t = 1,2,\dots$}
    \State Take \UPDATE{an} update using current $g_\ell$
    \If{\UPDATE{$t\ge t_{\mathrm w}$ and $t \bmod N = 0$}}
        \For{each $\ell \in L_{\text{elig}} \setminus L_{\text{lock}}$}
            \State \UPDATE{estimate a valid $\Rhat_\ell$}
            \State \UPDATE{$z_\ell\gets\mathrm{EMA}_\beta(\log\Rhat_\ell)$}
            \If{\UPDATE{$g_\ell=\textsc{Sign}$}}
                \State \UPDATE{$b_\ell\gets[z_\ell>\tau+m]$}
            \Else
                \State \UPDATE{$b_\ell\gets[z_\ell<\tau-m]$}
            \EndIf
            \State \UPDATE{$h_\ell\gets h_\ell+1$ if $b_\ell$, else $h_\ell\gets0$}
            \If{\UPDATE{$h_\ell\ge q$}}
                \State flip $g_\ell$;\, $L_{\text{lock}} \gets L_{\text{lock}} \cup \{\ell\}$
            \EndIf
        \EndFor
    \EndIf
\EndFor
\end{algorithmic}
\end{algorithm}

\subsection{\UPDATE{Directional Quadratic Score }\texorpdfstring{$\Rhat$}{R-hat}}

For a   \UPDATE{deterministic quadratic objective, a fixed } direction $D$,   \UPDATE{alignment $s(D):=\inner{G}{D}$, and positive directional curvature $C(D):=\inner{D}{\nabla^2 f(W)[D]}>0$,
}\[
 \UPDATE{f(W)-f(W-\eta D)=\eta s(D)-\tfrac12\eta^2C(D).
}\]
  \UPDATE{The best nonnegative line-search step is $\eta_D^*=(s(D))_+/C(D)$ and its exact decrease is $(s(D))_+^2/(2C(D))$. Motivated by this identity, the implementation averages stochastic alignment and HVP probes to obtain $\widehat s_g$ and $\widehat C_g$ for $g\in\{\mathrm{sign},\mathrm{spec}\}$ and computes
}\begin{equation}\UPDATE{\label{eq:empirical_index}
    \Rhat(W)
    \;:=\;
    \frac{(\widehat s_{\mathrm{spec}})_+^2 \,/\, \widehat C_{\mathrm{spec}}}{(\widehat s_{\mathrm{sign}})_+^2 \,/\, \widehat C_{\mathrm{sign}}},
    \qquad \widehat C_{\mathrm{spec}},\widehat C_{\mathrm{sign}}>0,
    \quad (\widehat s_{\mathrm{sign}})_+>0.
}\end{equation}
\UPDATE{For fixed directions on a deterministic positive-curvature quadratic, Eq.~\eqref{eq:empirical_index} equals the ratio of the two exact line-optimal decreases because their common factor $1/2$ cancels. In stochastic training, $\Rhat$ is a Monte Carlo local-quadratic routing score. It need not equal the global-smoothness ratio $\RLB$, because directional curvature need not attain the global constants; Appendix~\ref{app:diagnostic_details} gives the estimator details and safeguards.
}

\subsection{G-Scion: One-Way Per-Layer Switching}

\UPDATE{In the $31$M experiments, G-Scion uses $N{=}50$, $K{=}12$, nominal threshold $\tau{=}0$, log margin $m{=}0.1$, EMA factor $\beta{=}0.9$, and two consecutive confirmations. Eligible groups are forced to spectral during warmup and may then flip once; hidden layers remain spectral following the matched Scion design. The variants \emph{G-Scion (embedding)}, \emph{G-Scion (output)}, and \emph{G-Scion (embedding, output)} differ only in the eligible edge groups. Full settings are in Appendix~\ref{app:gscion}.}

\UPDATE{At each decision phase, G-Scion computes the directional score from signal, direction, and HVP probes, smooths its logarithm, and switches only after two margin-clearing confirmations (Algorithm~\ref{alg:g-scion}). One-way locking is an anti-oscillation safeguard. Diagnostic cost is amortized by the cadence and grows with the probe count and number of unlocked groups. A separate end-to-end timing experiment with $N{=}100$ measured approximately $4\%$ overhead for that configuration.}

\subsection{Relationship to Existing Mixed-Optimizer Recipes}
\UPDATE{Standard mixed-optimizer recipes fix geometry by layer type. Scion is the matched sign-versus-spectral baseline: replacing only its assignment mechanism isolates the binary decision studied here. Practical Muon recipes place AdamW on embedding and output groups and introduce momentum and coordinate-wise adaptivity, so we report AdamW+Muon as a strong practical reference rather than an upper bound or a like-for-like baseline. G-Scion replaces Scion's fixed assignment with the measured score $\Rhat$.}

\section{LM Pre-Training Experiments}
\label{sec:lm_validation}

We evaluate the layer- and stage-dependence predictions on GPT-2 pre-training, where shared hidden architecture makes per-layer diagnostics clean.

\paragraph{Setup.}
\UPDATE{We train a $31$M-parameter NanoGPT model on OpenWebText with four variants. \textbf{Scion} fixes SignSGD on embedding/output and Muon on hidden matrices. \textbf{G-Scion (embedding, output)}, \textbf{G-Scion (embedding)}, and \textbf{G-Scion (output)} use the gate of \secref{sec:method} on the indicated edge group(s); all other assignments match Scion.}

\paragraph{Results.}
\UPDATE{The like-for-like baseline is fixed Scion. Table~\ref{tab:gscion}(a) shows that the G-Scion assignment trajectory lowers validation loss by $0.355$ in the memoryless comparison and by $0.129$ with matched Signum momentum (GELU, three seeds). Table~\ref{tab:gscion}(b) decomposes the memoryless outcome across edge groups. These comparisons measure the performance of the resulting assignments. A separate five-seed follow-up compares all four static assignments and the gate (Appendix~\ref{app:assignment_followup}); it does not support a benefit over the best static assignment. AdamW+Muon is reported in Appendix~\ref{app:lm_owt} as a strong practical reference outside the memoryless comparison.}

\begin{table}[t]
\centering

\caption{\UPDATE{GPT-2/OWT, $31$M ($2{,}000$ steps, lr $10^{-3}$, effective batch $128$). (a) Resulting G-Scion versus fixed-Scion trajectories with memoryless SignSGD and with matched Signum momentum on the edge groups (GELU, $3$ seeds); Muon is fixed on hidden matrices. (b) Edge-group decomposition in the memoryless setting, averaged over GELU, ReLU$^2$, and SiLU at seed $1337$. $\Delta$ is the validation-loss reduction relative to the matched Scion row.}}
\label{tab:gscion}
\small

\noindent\textbf{(a)}~Final val loss, sign-vs-spectral by momentum regime.\\[0.25em]
\begin{tabular}{lcc}
\toprule
\textbf{Method} & \textbf{Final val loss} & \textbf{$\Delta$ vs.\ matched Scion} \\
\midrule
\multicolumn{3}{l}{\emph{w/o Momentum (SignSGD on embed/head, Muon on hidden)}} \\
Scion w/o Momentum                              & $5.529 \pm 0.012$ & --- \\
\textbf{G-Scion w/o Momentum}                    & $\mathbf{5.174 \pm 0.011}$ & $\mathbf{+0.355}$ \\
\midrule
\multicolumn{3}{l}{\emph{w/ Momentum (Signum on embed/head, Muon on hidden)}} \\
Scion w/ Momentum                               & $5.084 \pm 0.013$ & --- \\
\textbf{G-Scion w/ Momentum}                     & $\mathbf{4.955 \pm 0.004}$ & $\mathbf{+0.129}$ \\
\bottomrule
\end{tabular}\\
\medskip
\noindent\textbf{(b)}~Final val loss, w/o Momentum (mean over GELU, ReLU$^2$, SiLU; $\Rhat$).\\[0.25em]
\begin{tabular}{lrr}
\toprule
\textbf{Method} & \textbf{Final val loss} & \textbf{$\Delta$ vs.\ Scion} \\
\midrule
Scion w/o Momentum & $5.516$ & --- \\
G-Scion w/o Momentum (embedding) & $5.293$ & $+0.22$ \\
G-Scion w/o Momentum (output) & $5.447$ & $+0.07$ \\
\textbf{G-Scion w/o Momentum (embedding, output)} & $\mathbf{5.187}$ & $\mathbf{+0.33}$ \\
\bottomrule
\end{tabular}
\end{table}

\paragraph{\UPDATE{The $31$M output supplies a sign-side routing case.}}
\UPDATE{The eligible groups are forced to spectral during warmup. At the first measurement, embedding readings depend on the gated variant and are therefore not used as selector evidence. The reported six gated-output runs route the output to sign at $t{=}150$ and remain there. The gate uses two confirmations of an EMA-smoothed log score, rather than a pooled raw-score mean; Appendix~\ref{app:lm_owt} distinguishes these quantities and the limits of the historical summaries. The historical proxy-run comparison is not used as evidence because its exact normalization and gate configuration remain unverified; raw $\Rzer\le1$ cannot favor spectral at a unit threshold (Appendix~\ref{app:lm_owt}). The controlled same-iterate sweep in Appendix~\ref{app:exp4_kappa} separately isolates the added directional-curvature factor within $R_{\mathrm{quad}}=\Rzer\hat\kappa_{\mathrm{dir}}$. Table~\ref{tab:lr_output_logA} reports a phase-stratified sign change at the highest learning rate; it is an aggregate across runs, not a trace of a second switch.}

\begin{figure}[t]
\centering
\includegraphics[width=0.96\linewidth]{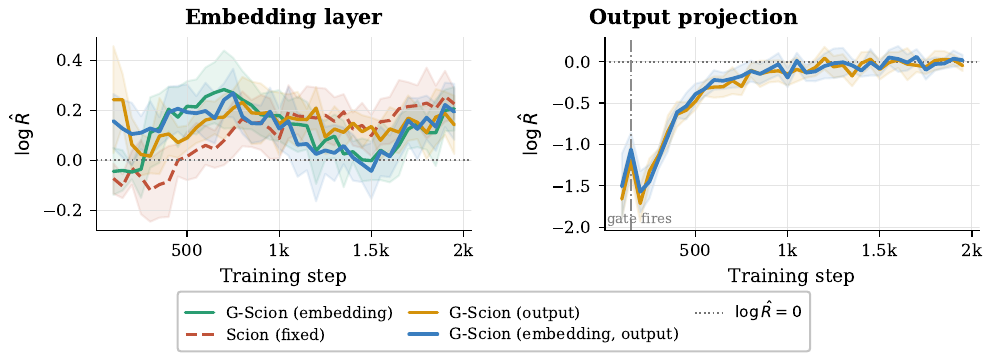}
\caption{\UPDATE{\textbf{Per-layer $\log\Rhat$ diagnostics} (OWT, $31$M, $2{,}000$ steps; mean $\pm$ std across three activations per method). \emph{Left:} embedding diagnostics are descriptive because eligible embeddings were already forced spectral during warmup and their first readings differ by variant. \emph{Right:} output-gated variants start spectral with a clearly negative score, satisfy the two-confirmation rule at $t{=}150$ (dashed), and route to sign. Scion is diagnostic only and does not gate.}}
\label{fig:lm_logA}
\end{figure}

\paragraph{Static assignments and stage dependence.}
Table~\ref{tab:assignment_followup} compares all four static edge assignments with G-Scion in a separate $31$M GPT-2/OWT follow-up. Static PP (spectral embedding and output) reaches a loss of $4.92$, versus $5.45$ for G-Scion. Early output-head recommendations favoring sign agree with the paired held-out decreases on three diagnostic seeds, but the spectral updates become better around step $700$, after the gate locks at step $150$ of $2{,}000$. The lock prevents adaptation to this later preference but it does not by itself establish the cause of the final-loss gap. A warm-up variant and a frozen-route replay both reach a loss of $5.23$ (rounded). These results distinguish the local score agreement from routing effectiveness as the gate does not outperform the best static assignment in this study.

\begin{table}[htbp]
\centering
\caption{Reported $31$M GPT-2/OWT follow-up. S denotes sign and P spectral; the embedding assignment is listed first. Lower validation loss is better. Values are rounded figure labels; per-seed values and the plotted error-bar definition are not available in the accompanying summary, so no uncertainty estimates or significance tests are inferred.}
\label{tab:assignment_followup}
\begin{tabular}{lc}
\toprule
Assignment or routing variant & Final validation loss \\
\midrule
PP: spectral on both edge groups & 4.92 \\
Frozen-route replay & 5.23 \\
G-Scion with warm-up & 5.23 \\
SP: spectral on output only & 5.35 \\
PS: spectral on embedding only & 5.40 \\
G-Scion & 5.45 \\
G-Scion with reversal allowed & 5.45 \\
SS: sign on both edge groups (fixed Scion) & 5.68 \\
\bottomrule
\end{tabular}
\end{table}

\paragraph{\UPDATE{Scale evidence.}}
\UPDATE{The available larger-scale runs do not form a matched score-and-assignment test: the $1$B panel used the curvature-free proxy, and several $163$M/$406$M method--seed cells lack complete directional records. We therefore treat scale dependence as a prediction and specify the matched directional-score, realised-route, and static-control comparison in Appendix~\ref{app:lm_scale_sweep}.}

  \section{\UPDATE{Cross-Architecture Assignment Evidence}}
\label{sec:vlm_experiments}

\paragraph{\UPDATE{ViT assignment outcome.}}
\UPDATE{We apply the same four edge-assignment variants to ViT-B/16 trained from scratch on CIFAR-100. At learning rate $3\times10^{-3}$, the joint spectral-edge trajectory exceeds fixed Scion by $11.9$, $11.5$, and $7.9$ validation-accuracy points at effective batches $128$, $256$, and $512$, respectively (Table~\ref{tab:vit_cifar_batch}). This shows that Scion's inherited sign assignment on embedding/output is not uniformly best outside language modelling. The ViT score readings are frequently nonfinite, however, so this experiment is assignment-outcome evidence rather than validation of $\Rhat$ or of online gating.}

\paragraph{\UPDATE{Weight-tied VLM scope.}}
\UPDATE{In nanoVLM~\citep{nanovlm2025}, the language embedding and output head are aliases of one shared parameter. A valid geometry comparison must therefore route the shared tensor exactly once and measure its joint gradient and HVP, rather than treating ``embedding'' and ``head'' as independent optimizer groups. The historical configurations did not realise the intended matched sign-versus-spectral head comparison and are excluded from the empirical claims. A corrected VLM evaluation requires explicit alias/identity logging, one realised optimizer route for the shared tensor, static sign and spectral controls, and a jointly measured shared-tensor score.}

 \section{Related Work}
\label{sec:related}

\paragraph{Spectral and sign geometries.}
Most modern optimizers~\citep{duchi2011adaptive,kingma2015adam,martens2015optimizing,gupta2018shampoo} are preconditioned Euclidean SGD; a complementary line targets the geometry of matrix-valued updates~\citep{carlson2015stochastic}, with Muon~\citep{jordan2024muon} as a recent instantiation that scales to LLMs~\citep{Liu25b,zeng2025glm}. On the sign side, SignSGD~\citep{bernstein2018signsgd} and Adam's entrywise-steepest interpretation~\citep{balles2020geometry} are the foundation; we use the sign-alignment lemma of \citet{balles2020geometry} in Eq.~\eqref{eq:sign_gain}, and the non-Euclidean noise scales of \citet{naganuma2026adaptivebatchsizesusing} cover finite batch size.

\paragraph{Stable rank and activations.}
\UPDATE{\citet{DavisDrusvyatskiy2025SpecGD} link a spectral-vs-GD quadratic-model comparison to gradient nuclear rank and post-activation stable rank, and bound the latter under explicit random-feature assumptions. Their result motivates measuring architecture-dependent curvature, but it does not determine the sign-vs-spectral ratio $L_\infty/L_\op$. Our extension to sign-vs-spectral is therefore a signal--curvature factorisation derived from the geometry-specific lower bounds together with an empirically evaluated directional score, not a corollary of their stable-rank theorem.
}

\paragraph{Activation-optimizer interactions.}
\UPDATE{Modern transformers use diverse activations~\citep{hendrycks2016gelu, ramachandran2017swish, so2021primer}, typically treated independently of optimizer choice. \citet{DavisDrusvyatskiy2025SpecGD} connect activation stable rank to a spectral-versus-Frobenius model comparison under stated conditions. We use that result as motivation rather than as a sign-versus-spectral guarantee. A causal activation test varies the activation alone while measuring the gradient, score, and realised route jointly.}

\section{Discussion}
\label{sec:discussion}

\paragraph{Single-step prediction and training outcomes.}
We use G-Scion's theoretically motivated signal to probe which geometry suits the current layer and training stage, separating two questions: does the score identify the better single step, and does following it improve the full run? The follow-up reports correct early sign recommendations on the output head and a later change in the better candidate (Appendix~\ref{app:assignment_followup}). This supports stage-dependent preferences, not universal score calibration. In the separate scheduled-update development study, the full-group score selects the better held-out step in 219/288 cases (76.0\%), versus 264/288 (91.7\%) for always spectral; mean regrets are 0.005037 and 0.003471 nats/token. The original self-curvature score achieves 205/284 on defined ratios (Appendix~\ref{app:scheduled_calibration}). These are single-step comparisons, not percentages of better training runs. For the second question, the follow-up's static spectral assignment beats the gate, which locks before the reported preference change. Warm-up gating and frozen replay have equal rounded endpoints. This motivates separating the measurement from the routing rule, but does not isolate the loss gap's cause; warm-up, smoothing and locking require controlled ablations.

\paragraph{Analytical and operational roles.}
\UPDATE{Proposition~\ref{prop:lower_bound_sandwich} turns the lower-bound comparison into an actual conditional expected-decrease ordering guarantee beyond the stated thresholds for ideal memoryless norm-steepest steps at a common iterate. G-Scion uses the related local directional score to route training. Trajectory dominance and joint optimality are separate questions because the implementation also includes stochastic estimation, scheduling, approximate polar factors, and simultaneous block updates. AdamW+Muon is a strong practical reference rather than a theoretical upper bound}~\citep{sato2025muonconvergence}.

\paragraph{Practical comparison and theory scope.}
The reported AdamW-edge/Muon reference attains validation loss $4.755\pm0.035$, compared with $4.851\pm0.016$ for G-Muon; lower is better. This negative practical comparison remains relevant even though AdamW is outside the memoryless theorem. Equal-budget tuning and complete state-dependent proposal tests are still required. At zero noise, the corrected bound certifies a spectral decision only for $\RLB>3$ and a sign decision only for $\RLB<1/3$; the interval between them is unresolved and widens with noise. No $\RLB$ values were estimated in the new LM study, so its directional-score bins are not theoretical certification regions. Frobenius pairwise validation remains incomplete; the analytical extension in Appendix~\ref{app:additional_theory} establishes neither empirical dominance nor inferiority.

\paragraph{Limitations.}
\UPDATE{(i) The layerwise score omits cross-block curvature, so single-block predictions need not identify the best joint assignment.
(ii) One-way locking can preserve an obsolete assignment. The follow-up's best static route beats its gate; allowing reversal alone does not recover a switch. The proposed roles of locking, smoothing and warm-up require matched ablations.
(iii) Directional-score evidence is currently concentrated at $31$M; scale and VLM generalisation require the matched comparisons described above.
(iv) $\Rhat$ requires Hessian--vector products. The measured overhead was approximately $4\%$ in a separate $N{=}100$ timing experiment and varies with cadence, probes, unlocked groups, model, and implementation.
(v) ViT currently provides assignment-outcome evidence rather than validation of its directional-score trace.
(vi) The follow-up reports five seeds for four static routes and the gate, and three for checkpoint diagnostics; other cells use two or three. Broader robustness and threshold calibration remain open.}

\section{Conclusion}
\label{sec:conclusion}

G-Scion studies measured sign-versus-spectral assignment. The finite-batch lower-bound ratio provides sufficient expected-decrease conditions for specified conservative memoryless steps; it does not certify the implemented gate at a threshold of one. The directional score has a line-optimal quadratic interpretation, but our paired scheduled-update development results show substantial ranking errors. An actual-displacement predictor is promising on the same development data and remains unconfirmed. Reported gains over fixed Scion do not establish superiority over the best static assignment or AdamW/Muon. The present evidence therefore supports a diagnostic framework and motivates independent calibration and controlled assignment studies, rather than a general claim of optimal routing or practical optimizer superiority.

\subsection*{AI use statement}
Generative AI was used to polish and condense the writing.

\subsection*{Ethics statement}
The work studies optimizer geometry and training dynamics using public datasets, public model architectures, and controlled training experiments. It does not involve human subjects, crowdsourcing, private user data, deception, or deployment of high-risk systems. The paper's claims are technical and are framed within the limitations of the experiments.

\subsection*{Reproducibility statement}
\UPDATE{The paper specifies the datasets, model families and sizes, optimizer variants, edge groups, cadence $N$, probe count $K$, margin, EMA, confirmation rule, learning rates, batch sizes, training steps, seeds, and the alignment/HVP quantities used by the directional score. Section~\ref{app:experiments} records the provenance limits of the larger-scale, ViT-score, and historical VLM artifacts. Code-level reproduction is limited by the incomplete submitted repository, as acknowledged in Appendix~\ref{app:code_availability}; the follow-up's per-seed records and full configuration mapping also remain to be archived (Appendix~\ref{app:assignment_followup}).}

\ificlrfinal
\subsubsection*{Acknowledgements}

The authors also acknowledge the support of Compute Canada and Mila computing clusters for experimental resources.
\fi

\bibliographystyle{iclr2027_conference}
\bibliography{references}

\clearpage
\appendix

\section*{Appendix Contents}
\begin{tabular}{@{}p{0.91\linewidth}r@{}}
\ref{app:notation}\quad\hyperref[app:notation]{Notation} & \pageref{app:notation} \\
\ref{app:convergence_theory}\quad\hyperref[app:convergence_theory]{Convergence Theory} & \pageref{app:convergence_theory} \\
\ref{app:is_bounds}\quad\hyperref[app:is_bounds]{$\Rzer$ Approximation Bounds} & \pageref{app:is_bounds} \\
\ref{app:activation_theory}\quad\hyperref[app:activation_theory]{Activation-Dependent Stable Rank Theory} & \pageref{app:activation_theory} \\
\ref{app:experiments}\quad\hyperref[app:experiments]{Experimental Details} & \pageref{app:experiments} \\
\ref{app:related_work}\quad\hyperref[app:related_work]{Comprehensive Related Work} & \pageref{app:related_work} \\
\ref{app:compute}\quad\hyperref[app:compute]{Compute Resources} & \pageref{app:compute} \\
\ref{app:scheduled_calibration}\quad\hyperref[app:scheduled_calibration]{Held-out Scheduled-Update Calibration} & \pageref{app:scheduled_calibration} \\
\ref{app:assignment_followup}\quad\hyperref[app:assignment_followup]{Static Assignments and Stage-Dependent Preferences} & \pageref{app:assignment_followup} \\
\ref{app:additional_theory}\quad\hyperref[app:additional_theory]{Additional Theoretical Details} & \pageref{app:additional_theory} \\
\end{tabular}

\section{Notation}
\label{app:notation}

This appendix collects the notation used throughout the paper, grouped by category. All definitions are consistent with the main-text usage in \secref{sec:theory}.

\subsection{Parameters, gradients, and batch quantities}
\label{app:notation_grad}

\begin{table}[H]
\centering
\small
\begin{tabular}{@{}>{\raggedright\arraybackslash}p{0.18\linewidth}>{\raggedright\arraybackslash}p{0.76\linewidth}@{}}
\toprule
\textbf{Symbol} & \textbf{Meaning} \\
\midrule
$W \in \R^{m\times n}$ & weight matrix of a given layer \\
$m, n$ & output / input dimension of the layer \\
$p:= \min(m, n)$ & smaller side of the matrix \\
$f: \R^{m\times n} \to \R$ & (possibly regularised) training loss at parameter $W$ \\
$G:= \nabla f(W)$ & full-batch (population) gradient at $W$ \\
$G = U \Sigma V^\top$ & thin SVD of $G$, with $\Sigma \in \R^{p\times p}$ \\
$\widehat{G} = G + E$ & minibatch gradient estimate; $E$ is mean-zero additive noise \\
$B$ & minibatch (or effective-batch) size used to form $\widehat{G}$ \\
$K$ & number of independent minibatch gradients used by $\Rhat$ (\UPDATE{$12$ in the reported $31$M runs}) \\
$N$ & decision cadence of the G-Scion gate (\UPDATE{$50$ steps in the reported $31$M runs}) \\
$t$ & training step index \\
$\tau$ & \UPDATE{nominal} switching threshold in $\log \Rhat$ ($0$) \\
\UPDATE{$m$} & \UPDATE{log-score margin ($0.1$ in the reported $31$M runs)} \\
\UPDATE{$\beta$} & \UPDATE{EMA factor ($0.9$ in the reported $31$M runs)} \\
\UPDATE{$q$} & \UPDATE{required consecutive confirmations ($2$ in the reported $31$M runs)} \\
\bottomrule
\end{tabular}
\end{table}

\subsection{Norms, inner products, and matrix operators}
\label{app:notation_norms}

\begin{table}[H]
\centering
\small
\begin{tabular}{@{}p{0.30\linewidth}p{0.66\linewidth}@{}}
\toprule
\textbf{Symbol} & \textbf{Meaning} \\
\midrule
$\inner{A}{B}:= \tr(A^\top B)$ & Frobenius inner product \\
$\norm{\cdot}_\F$ & Frobenius norm, $\norm{A}_\F = \sqrt{\inner{A}{A}}$ \\
$\specnorm{\cdot} = \norm{\cdot}_{\op}$ & spectral (operator) norm, largest singular value \\
$\nucnorm{\cdot} = \norm{\cdot}_{\nuc}$ & nuclear (trace) norm, sum of singular values \\
$\onenorm{\cdot} = \norm{\cdot}_{1}$ & entrywise $\ell_1$ norm, $\sum_{ij}|A_{ij}|$ \\
$\norm{\cdot}_\infty$ & entrywise $\ell_\infty$ norm, $\max_{ij}|A_{ij}|$ \\
$\norm{\cdot}_*$ & dual norm of $\norm{\cdot}$ \\
$\sign(G)$ & entrywise sign of $G$ (with $\sign(0):=0$) \\
$\polar(G) = U V^\top$ & polar factor of $G$ (partial isometry from its SVD) \\
$\tr(\cdot)$ & matrix trace \\
$\rank(\cdot)$ & matrix rank \\
\bottomrule
\end{tabular}
\end{table}

The three geometries of interest correspond to the pairs
$(\norm{\cdot}, \norm{\cdot}_*) = (\norm{\cdot}_\F, \norm{\cdot}_\F)$ for (S)GD,
$(\norm{\cdot}_\infty, \onenorm{\cdot})$ for SignSGD,
and $(\specnorm{\cdot}, \nucnorm{\cdot})$ for SpecSGD; see \tabref{tab:geometries}.

\subsection{Smoothness constants and the curvature ratio}
\label{app:notation_smoothness}

\begin{table}[H]
\centering
\small
\begin{tabular}{@{}ll@{}}
\toprule
\textbf{Symbol} & \textbf{Meaning} \\
\midrule
$L_{\norm{\cdot}}$ & smoothness constant of $f$ w.r.t.\ a primal norm $\norm{\cdot}$ \\
$L_\F$ & Frobenius smoothness constant \\
$\Lop$ & operator (spectral) smoothness constant \\
$\Linf$ & entrywise $\ell_\infty$ smoothness constant \\
$\kappa:= \Linf / \Lop$ & \emph{curvature ratio} (central configuration-side quantity) \\
$C(D):= \inner{D}{\nabla^2 f(W)[D]}$ & directional curvature along unit direction $D$ \\
\bottomrule
\end{tabular}
\end{table}

The smoothness constants satisfy $\Lop \le p \cdot L_\F$ and $L_\F \le \Linf$
(Appendix~\ref{app:smoothness}, \Propref{prop:smoothness_relations}).

\subsection{Update directions and one-step improvements}
\label{app:notation_directions}

\begin{table}[H]
\centering
\small
\begin{tabular}{@{}ll@{}}
\toprule
\textbf{Symbol} & \textbf{Meaning} \\
\midrule
$D^*(G) \in \argmax_{\norm{D}\le 1} \inner{G}{D}$ & norm-steepest-descent direction for primal norm $\norm{\cdot}$ \\
$\Dsign:= \sign(\widehat{G})$ & SignSGD update direction \\
$\Dspec:= \polar(\widehat{G})$ & SpecSGD / Muon-like update direction \\
$W^+$ & next iterate after an ideal scaled step from $W$ \\
\UPDATE{$\Delta_g(W;\eta)$} & \UPDATE{conditional expected one-step decrease under geometry $g$ at step $\eta$} \\
\UPDATE{$B_{\mathrm{sign}}(W)$} & \UPDATE{SignSGD conservative lower bound (Eq.~\eqref{eq:sign_gain})} \\
\UPDATE{$B_{\mathrm{spec}}(W)$} & \UPDATE{SpecSGD conservative lower bound (Eq.~\eqref{eq:spec_gain})} \\
\bottomrule
\end{tabular}
\end{table}

\subsection{Noise scales}
\label{app:notation_noise}

\begin{table}[H]
\centering
\small
\begin{tabular}{@{}ll@{}}
\toprule
\textbf{Symbol} & \textbf{Meaning} \\
\midrule
$E:= \widehat{G} - G$ & stochastic gradient noise matrix \\
$\nu_1(W):= \E[\onenorm{\widehat{G} - G}\mid W]$ & $\ell_1$ (sign-geometry) noise scale \\
$\nu_\nuc(W):= \E[\nucnorm{E}\mid W]$ & nuclear-norm (spectral-geometry) noise scale \\
$c_1:= \sqrt{B}\cdot \nu_1$ & batch-size-normalised $\ell_1$ noise scale \\
$c_\nuc:= \sqrt{B}\cdot \nu_\nuc$ & batch-size-normalised nuclear noise scale \\
$\nrank{G}:= \nucnorm{G}^2 / \norm{G}_\F^2$ & nuclear rank (effective rank in the $\nuc$-vs-$\F$ sense) \\
$\phi_{\mathrm{ent}}(G):= \onenorm{G}^2 / (mn\,\norm{G}_\F^2)$ & entrywise density of $G$ \\
\bottomrule
\end{tabular}
\end{table}

\subsection{\UPDATE{Global Lower-Bound Ratios and Directional Scores}}
\label{app:notation_index}

The paper \UPDATE{distinguishes the ratios derived from global-smoothness lower bounds from the practical directional score} (main text, \secref{sec:theory}; Definition~\ref{def:advantage_index}):

\begin{table}[H]
\centering
\small
\begin{tabular}{@{}ll@{}}
\toprule
\textbf{Symbol} & \textbf{Meaning} \\
\midrule
  \UPDATE{$\Delta_g(W;\eta)$ } &   \UPDATE{actual expected decrease for geometry $g$ at step size $\eta$ } \\
  \UPDATE{$B_g(W,B)$ } &   \UPDATE{lower-bound quantity for the stated conservative step }\\
\UPDATE{$\RLB(W,B)$ }&  finite-batch geometry ratio \UPDATE{$B_{\mathrm{spec}}/B_{\mathrm{sign}}$} (Eq.~\eqref{eq:advantage_index}) \\
$\Rinft(W)$ & high-SNR \UPDATE{form of the global-smoothness ratio (Eq.~\eqref{eq:high_snr})} \\
$\Rzer(W):= \nucnorm{G}^2 / \onenorm{G}^2$ & curvature-free proxy (Eq.~\eqref{eq:is_def}) \\
$\Rhat(W)$ &   \UPDATE{empirical directional quadratic score } (Eq.~\eqref{eq:empirical_index}) \\
\bottomrule
\end{tabular}
\end{table}

  \UPDATE{In the positive-signal regime, $\RLB=\Rinft\cdot[\text{noise-correction factor}]$ and $\Rinft=\Rzer\cdot\kappa$. The directional score $\Rhat$ is not generally an estimator of either global-smoothness quantity.
}

\subsection{Activation-dependent stable rank}
\label{app:notation_activation}

\begin{table}[H]
\centering
\small
\begin{tabular}{@{}ll@{}}
\toprule
\textbf{Symbol} & \textbf{Meaning} \\
\midrule
$\sigma(\cdot)$ & pointwise activation function \\
$X_\ell$ & pre-activation at layer $\ell$: $X_\ell = W_\ell A_{\ell-1}$ \\
$A_\ell:= \sigma(X_\ell)$ & post-activation at layer $\ell$ \\
$\strank{A}:= \norm{A}_\F^2 / \specnorm{A}^2$ & stable rank of a matrix $A$ \\
$\effrank{A}$ & effective rank (alternative rank surrogate; see Appendix~\ref{app:activation_theory}) \\
$\ssigma:= \E[\sigma(\gamma)^2] / (\E[\sigma(\gamma)])^2$ & activation noise-to-signal ratio, $\gamma\sim\mathcal{N}(0,1)$ \\
$\NSR$ & noise-to-signal ratio (synonym for $\ssigma$ when context is clear) \\
\bottomrule
\end{tabular}
\end{table}

Proposition~\ref{prop:strank_activation}   \UPDATE{bounds $\strank{\sigma(WX)}$ by an activation moment ratio, subject to the stated random-feature assumptions. This controls the spectral--Frobenius feature-majorizer ratio only; it supplies no bound on the sign--spectral global ratio $\kappa=L_\infty/L_\op$}.

\subsection{Optimizers and methods}
\label{app:notation_methods}

The paper compares three primal geometries for matrix-valued parameters, two named single-geometry baselines, and the mixed-optimizer recipes built on top of them. The table below collects the abbreviations used in the main text and figures.

\begin{table}[H]
\centering
\small
\begin{tabular}{@{}>{\raggedright\arraybackslash}p{0.18\linewidth}>{\raggedright\arraybackslash}p{0.76\linewidth}@{}}
\toprule
\textbf{Name} & \textbf{Meaning} \\
\midrule
(S)GD & (stochastic) gradient descent; Frobenius-norm steepest descent \\
SignSGD, \SignSGD & sign-based steepest descent under entrywise $\ell_\infty$ geometry \\
SpecSGD, \SpecSGD & spectral steepest descent under operator-norm geometry; direction $\polar(\widehat{G})$ \\
Muon & practical spectral-geometry optimizer used as our SpecSGD instantiation \\
AdamW & Adam with decoupled weight decay (standard baseline) \\
Scion & fixed mixed recipe: SignSGD on embed / head, Muon on hidden layers \\
G-Scion & adaptive gate: \UPDATE{after forced warmup, an eligible layer may switch once after margin, EMA, and confirmation checks driven by $\Rhat$} \\
G-Muon & variant of G-Scion with AdamW (not SignSGD) as the ``sign side'' \\
\bottomrule
\end{tabular}
\end{table}

\subsection{Standard mathematical notation}
\label{app:notation_misc}

We follow standard probability and analysis conventions; symbols below are listed for completeness.

\begin{table}[H]
\centering
\small
\begin{tabular}{@{}ll@{}}
\toprule
\textbf{Symbol} & \textbf{Meaning} \\
\midrule
$\R, \mathbb{N}, \mathbb{Z}$ & real numbers, natural numbers, integers \\
$\E[\cdot]$ & expectation (over minibatch sampling unless specified otherwise) \\
$\Var[\cdot], \Cov[\cdot,\cdot]$ & variance, covariance \\
$\mathcal{N}(\mu, \Sigma)$ & Gaussian distribution with mean $\mu$ and covariance $\Sigma$ \\
$\mathbb{1}[\cdot]$ & indicator function \\
$(x)_+:= \max(x, 0)$ & non-negative part \\
$O(\cdot), \Theta(\cdot), o(\cdot)$ & standard asymptotic notation \\
$\log(\cdot)$ & natural logarithm unless otherwise stated \\
\bottomrule
\end{tabular}
\end{table}

\newpage
\section{Convergence Theory}
\label{app:convergence_theory}

\subsection{Smoothness Relations}
\label{app:smoothness}

We establish the relationships between the  \UPDATE{least admissible } norm-specific  \UPDATE{global } smoothness constants used throughout the paper. Recall that $f: \R^{m \times n} \to \R$ is $L_{\norm{\cdot}}$-smooth w.r.t.\ a norm $\norm{\cdot}$ if
$\norm{\nabla f(X) - \nabla f(Y)}_* \le L_{\norm{\cdot}} \norm{X - Y}$ for all $X, Y$.

\begin{proposition}[Norm-specific smoothness relations]
\label{prop:smoothness_relations}
  \UPDATE{Let $p=\min(m,n)$ and $d=mn$. Then
}\begin{align}
\UPDATE{L_\F }&\UPDATE{\le L_\op \le pL_\F, \label{eq:Lop_bound} }\\
\UPDATE{L_\op }&\UPDATE{\le L_\infty \le dL_\F. \label{eq:LF_bound}
}\end{align}
  \UPDATE{Consequently,
}\begin{equation}
\UPDATE{1\le\kappa:=\frac{L_\infty}{L_\op}\le mn.
}\end{equation}
 \end{proposition}

\begin{proof}
  \UPDATE{Use the norm relations
}\[
\UPDATE{\norm{A}_\F\le\norm{A}_\nuc\le\sqrt p\,\norm{A}_\F,
\qquad
\norm{A}_\nuc\le\norm{A}_1,
}\]
\UPDATE{and
}\[
\UPDATE{\norm{M}_\infty\le\norm{M}_\op\le\norm{M}_\F
\le\sqrt{mn}\,\norm{M}_\infty.
}\]
\UPDATE{For example, $L_\op\le L_\infty$ follows from
}\[
\UPDATE{\norm{\Delta\nabla f}_\nuc
\le\norm{\Delta\nabla f}_1
\le L_\infty\norm{\Delta W}_\infty
\le L_\infty\norm{\Delta W}_\op.
}\]
  \UPDATE{Applying the remaining norm comparisons on the input and output sides gives $L_\F\le L_\op\le pL_\F$ and $L_\infty\le mnL_\F$.
}\end{proof}

  \UPDATE{The upper bound is dimensionally sharp: for $f(W)=\tfrac12\inner{J}{W}^2$ with $J$ } the   \UPDATE{all-ones matrix, $L_\infty/L_\op=mn$. Global smoothness depends on the Hessian of the objective, not on the rank of the current gradient}.

 \subsection{Proof of the Polar Misalignment Lemma}
\label{app:polar_proof}

The following lemma bounds the alignment loss when computing the polar factor from a noisy gradient.

\begin{lemma}[Polar misalignment]
\label{lem:polar_misalignment}
Let $G = U \Sigma V^\top$ be the SVD of the true gradient and $\widehat{G} = G + E$ the noisy observation.
Then:
\begin{equation}
\inner{G}{\polar(\widehat{G})} \ge \nucnorm{G} - 2\nucnorm{E}.
\end{equation}
\end{lemma}

\begin{proof}
Write $\polar(\widehat{G}) = \widehat{U}\widehat{V}^\top$ where $\widehat{G} = \widehat{U}\widehat{\Sigma}\widehat{V}^\top$ is the SVD of $\widehat{G}$.
By the definition of the polar factor as the maximizer of the inner product over the set of partial isometries:
\begin{equation}
\inner{\widehat{G}}{\polar(\widehat{G})} = \nucnorm{\widehat{G}} \ge \inner{\widehat{G}}{UV^\top} = \inner{G}{UV^\top} + \inner{E}{UV^\top}.
\end{equation}
Thus:
\begin{align}
\inner{G}{\polar(\widehat{G})} &= \inner{\widehat{G}}{\polar(\widehat{G})} - \inner{E}{\polar(\widehat{G})} \\
&= \nucnorm{\widehat{G}} - \inner{E}{\widehat{U}\widehat{V}^\top}.
\end{align}

Now, by the triangle inequality for nuclear norms:
\[
\nucnorm{\widehat{G}} = \nucnorm{G + E} \ge \nucnorm{G} - \nucnorm{E}.
\]
Also, $|\inner{E}{\widehat{U}\widehat{V}^\top}| \le \nucnorm{E} \cdot \specnorm{\widehat{U}\widehat{V}^\top} = \nucnorm{E}$,
since $\specnorm{\widehat{U}\widehat{V}^\top} = 1$ (it is a partial isometry).
Combining:
\begin{align}
\inner{G}{\polar(\widehat{G})} &\ge \nucnorm{G} - \nucnorm{E} - \nucnorm{E} = \nucnorm{G} - 2\nucnorm{E}.
\end{align}
\end{proof}

This bound is used in deriving the SpecSGD expected improvement (Eq.~\ref{eq:spec_gain} in the main text).
The factor of two arises because noise affects both the choice of direction (through the polar factor) and the inner product evaluation.

\subsection{  \UPDATE{Ideal Arbitrary-Switching } Norm-Steepest Descent}
\label{app:convergence}

We   \UPDATE{record a convergence guarantee for an idealized deterministic method that may choose either the sign or spectral geometry at every iteration. The choice is arbitrary; the result does not require a gate or a } one-step   \UPDATE{score.
}

  \begin{theorem}[\UPDATE{Ideal arbitrary-switching descent and stationarity}]
\label{thm:ideal_arbitrary_switching}
\UPDATE{Let $f:\R^{m\times n}\to\R$ be differentiable and bounded below by $f_{\inf}>-\infty$. Assume that $f$ is $L_\infty$-smooth with respect to the entrywise $\ell_\infty$ norm and $L_\op$-smooth with respect to the spectral norm. At iteration $t$, choose arbitrarily
}\[
\UPDATE{(\norm{\cdot}_t,\norm{\cdot}_{t,*},L_t)
\in
\left\{
(\norm{\cdot}_\infty,\onenorm{\cdot},L_\infty),
(\specnorm{\cdot},\nucnorm{\cdot},L_\op)
\right\}.
}\]
  \UPDATE{Let $G_t=\nabla f(W_t)$ and define
}\[
\UPDATE{D_t\in\argmax_{\norm{D}_t\le1}\inner{G_t}{D},
\qquad
\eta_t:=\frac{\norm{G_t}_{t,*}}{L_t},
\qquad
W_{t+1}:=W_t-\eta_tD_t.
}\]
\UPDATE{Then, for every } $t$,
  \begin{equation}
\UPDATE{f(W_{t+1})
\le
f(W_t)-\frac{\norm{G_t}_{t,*}^{2}}{2L_t}.
\label{eq:ideal_switch_descent}
}\end{equation}
  \UPDATE{Consequently, } with   \UPDATE{$L_{\max}:=\max\{L_\infty,L_\op\}$, for every $T\ge1$,
}\begin{align}
\UPDATE{\sum_{t=0}^{T-1}\norm{\nabla f(W_t)}_\F^2
}&\UPDATE{\le
2L_{\max}\bigl(f(W_0)-f_{\inf}\bigr),
\label{eq:ideal_switch_sum}
}\\
\UPDATE{\min_{0\le t<T}\norm{\nabla f(W_t)}_\F^2
}&\UPDATE{\le
\frac{2L_{\max}\bigl(f(W_0)-f_{\inf}\bigr)}{T}.
\label{eq:ideal_switch_stationarity}
}\end{align}
  \UPDATE{In particular, $\norm{\nabla f(W_t)}_\F\to0$ as $t\to\infty$.
} \end{theorem}

\begin{proof}
  \UPDATE{Norm-specific smoothness gives
}\[
\UPDATE{f(W_t-\eta D_t)
\le
f(W_t)-\eta\inner{G_t}{D_t}
+\frac{L_t\eta^2}{2}\norm{D_t}_t^2.
}\]
By   \UPDATE{construction, $\norm{D_t}_t\le1$ and $\inner{G_t}{D_t}=\norm{G_t}_{t,*}$. Substituting $\eta=\eta_t$ proves Eq.~\eqref{eq:ideal_switch_descent}. For both candidate geometries,
}\[
\UPDATE{\onenorm{G_t}\ge\norm{G_t}_\F,
\qquad
\nucnorm{G_t}\ge\norm{G_t}_\F.
}\]
  \UPDATE{Since $L_t\le L_{\max}$, summing Eq.~\eqref{eq:ideal_switch_descent} and using $f(W_T)\ge f_{\inf}$ proves Eqs.~\eqref{eq:ideal_switch_sum}--\eqref{eq:ideal_switch_stationarity}. Finiteness of the infinite sum of squared gradient norms gives the final claim.
} \end{proof}

  \begin{remark}[\UPDATE{Scope of the result}]
\label{rem:ideal_switch_scope}
  \UPDATE{Theorem~\ref{thm:ideal_arbitrary_switching} concerns an ideal deterministic single-block update with exact norm-steepest directions and geometry-specific step sizes $\eta_t=\norm{G_t}_{t,*}/L_t$. It does not analyze the implemented G-Scion gate, which uses stochastic gradient and curvature estimates, finite decision cadence, one-way locking, shared learning-rate schedules, approximate polar directions, and simultaneous updates of coupled parameter blocks. The theorem therefore does not imply that the gate selects the better realised update, that G-Scion dominates either fixed optimizer, or that the per-layer decisions are jointly optimal}.
\end{remark}

\section{\texorpdfstring{$\Rzer$}{R0} Approximation Bounds}
\label{app:is_bounds}

We derive upper and lower bounds on the \UPDATE{finite-batch geometry ratio} $\RLB(W,B)$ (Eq.~\ref{eq:advantage_index}) in terms of the curvature-free proxy $\Rzer$ (Eq.~\ref{eq:is_def}).

\subsection{Exact Decomposition of \texorpdfstring{$\RLB$}{R-LB}}

\subsubsection{Exact decomposition}
\label{app:exact_decomp}

Define the relative noise ratios:
\begin{equation}
\alpha:= \frac{2c_\nuc}{\sqrt{B}\,\nucnorm{G}}, \qquad
\beta:= \frac{2c_1}{\sqrt{B}\,\onenorm{G}},
\label{eq:relative_noise}
\end{equation}
and the curvature ratio $\kappa:= L_\infty / L_\op$.

\begin{proposition}[Exact decomposition]
\label{prop:decomposition}
Assuming $0 \le \alpha, \beta < 1$ (signal dominates noise):
\begin{equation}
\RLB(W, B) \;=\; \Rzer(G)
\;\cdot\; \underbrace{\left(\frac{1-\alpha}{1-\beta}\right)^{\!2}}_{\text{noise correction}}
\;\cdot\; \underbrace{\kappa\vphantom{\left(\frac{1}{1}\right)}}_{\text{curvature ratio}}.
\label{eq:exact_decomposition}
\end{equation}
\end{proposition}

\begin{proof}
Factor the noise-corrected norms as
$\nucnorm{G} - 2c_\nuc/\sqrt{B} = \nucnorm{G}(1 - \alpha)$ and
$\onenorm{G} - 2c_1/\sqrt{B} = \onenorm{G}(1 - \beta)$.
Substituting into Eq.~\eqref{eq:advantage_index}:
\[
\RLB(W,B)
= \left(\frac{\nucnorm{G}}{\onenorm{G}}\right)^{\!2}
  \left(\frac{1-\alpha}{1-\beta}\right)^{\!2} \kappa.
\]
From Eq.~\eqref{eq:is_def},
$\nucnorm{G}/\onenorm{G} = \sqrt{\Rzer}$,
giving~\eqref{eq:exact_decomposition}.
\end{proof}

This decomposition reveals a three-factor structure: $\Rzer$ captures gradient geometry, $\alpha$ and $\beta$ encode finite-batch noise, and $\kappa$   \UPDATE{is the ratio of global norm-smoothness constants}.

\subsubsection{High-SNR approximation}
\label{app:high_snr}

When the signal dominates noise ($\alpha \ll 1$ and $\beta \ll 1$), the noise correction factor satisfies
$\bigl(\frac{1-\alpha}{1-\beta}\bigr)^2 \approx 1$,
so:
\begin{equation}
\RLB(W, B) \;\approx\; \Rzer(G)\;\kappa.
\label{eq:high_snr_app}
\end{equation}
If the curvature ratio is also approximately unity ($\kappa \approx 1$):
\begin{equation}
\RLB(W, B) \;\approx\; \Rzer(G).
\label{eq:Is_limit}
\end{equation}
This is the regime in which $\Rzer$ \UPDATE{numerically approximates the global-smoothness geometry ratio. If $\kappa$ is merely common rather than close to one,} $\Rzer$ \UPDATE{preserves rankings but not magnitudes}.

\subsubsection{Upper and lower bounds}
\label{app:simple_bounds}

Since $0 \le \alpha, \beta < 1$, we have $1-\beta \le 1$ and hence $1/(1-\beta)^2 \ge 1$, giving:
\[
(1-\alpha)^2
\;\le\;
\left(\frac{1-\alpha}{1-\beta}\right)^{\!2}
\;\le\;
\frac{1}{(1-\beta)^2}.
\]
Combining with the exact decomposition:

\begin{corollary}[Upper and lower bounds]
\label{cor:bounds}
\begin{equation}
\kappa\,(1-\alpha)^2\,\Rzer
\;\;\le\;\;
\RLB(W, B)
\;\;\le\;\;
\frac{\kappa\,\Rzer}{(1-\beta)^2}.
\label{eq:Is_bounds}
\end{equation}
\end{corollary}

\paragraph{Reading the bounds.}
\UPDATE{Relative to $\kappa\Rzer$, the noise factor $((1-\alpha)/(1-\beta))^2$ is below one exactly when $\alpha>\beta$. A large spectral-side relative noise $\alpha$ shrinks the numerator, but does not alone imply $\RLB<\Rzer$; that comparison also depends on $\beta$ and $\kappa$. As the sign-side relative noise $\beta$ approaches $1$, the denominator can nearly vanish and the factor can diverge.}

\paragraph{Curvature-free bounds.}
The norm equivalences from Proposition~\ref{prop:smoothness_relations} yield:
  \begin{equation}
\UPDATE{1 \;\le\; \kappa = \frac{L_\infty}{L_\op} \;\le\; mn.
\label{eq:kappa_bounds}
}\end{equation}
Substituting into~\eqref{eq:Is_bounds} eliminates the curvature ratio entirely:
  \begin{equation}
\UPDATE{(1-\alpha)^2\,\Rzer
\;\;\le\;\;
\RLB(W, B)
\;\;\le\;\;
\frac{mn\,\Rzer}{(1-\beta)^2}.
\label{eq:curvature_free_bounds}
}\end{equation}
In the noiseless limit ($\alpha, \beta \to 0$):
  \begin{equation}
\UPDATE{\Rzer
\;\le\;
\RLB(W, B)
\;\le\;
mn\,\Rzer.
\label{eq:noiseless_bounds}
}\end{equation}
The upper bound shows that \UPDATE{$\Rzer>1/(mn)$ is necessary for $\RLB>1$ in the noiseless regime. Because $\Rzer\le1$, norm equivalence alone supplies no nontrivial curvature-free sufficient condition that is uniformly strict. These inequalities compare the lower-bound quantities, not realised optimizer decreases}.

\subsubsection{First-order noise correction}
\label{app:first_order}

For $\alpha, \beta \ll 1$, a Taylor expansion gives:
\[
\frac{1-\alpha}{1-\beta}
= (1-\alpha)(1+\beta+O(\beta^2))
= 1 - (\alpha - \beta) + O(\alpha\beta + \beta^2),
\]
so:
  \begin{equation}
\RLB(W, B)
\;\approx\;
\kappa\,\Rzer
\,\bigl(1 - 2(\alpha - \beta)\bigr).
\label{eq:first_order}
\end{equation}

\paragraph{Reading the first-order correction.}
When   \UPDATE{$\alpha>\beta$}, the spectral-side relative noise exceeds the sign-side and $\RLB$ falls below the noiseless proxy   \UPDATE{$\kappa\Rzer$; when $\beta>\alpha$ the situation reverses}.  \UPDATE{No universal monotonic dependence on batch size follows without an assumption on the ordering and batch dependence of these two relative noise terms.
}

\subsubsection{When is \texorpdfstring{$\Rzer$}{R0} alone sufficient?}
\label{app:validity}

The proxy $\Rzer$   \UPDATE{numerically approximates } $\RLB$ when   \UPDATE{the following } conditions hold:
\begin{enumerate}
\item \textbf{High SNR in both geometries:}
$\sqrt{B}\,\nucnorm{G} \gg 2c_\nuc$ and $\sqrt{B}\,\onenorm{G} \gg 2c_1$
(i.e., $\alpha \ll 1$ and $\beta \ll 1$).
\item \textbf{Comparable norm-specific smoothness constants:}
$L_\infty \approx L_\op$
(i.e., $\kappa \approx 1$).
 \end{enumerate}
Under these conditions, \UPDATE{$\RLB(W,B)\approx\Rzer(G)$ as a lower-bound ratio}.

Conversely, $\Rzer$ alone can be misleading when:
\begin{itemize}
\item the batch size is small (large $\alpha$ or $\beta$), so that noise corrections dominate the signal;
\item the gradient signal is weak (small $\nucnorm{G}$ or $\onenorm{G}$);
\item the curvature asymmetry is large (\UPDATE{$\kappa\gg1$}).
\end{itemize}

\paragraph{Three-factor decomposition.}
Eq.~\eqref{eq:exact_decomposition} decomposes $\RLB$ into three factors:
  \[
\UPDATE{\RLB(W, B)
= \underbrace{\Rzer(G)}_{\text{gradient structure}}
\;\cdot\; \underbrace{\left(\frac{1-\alpha}{1-\beta}\right)^{\!2}}_{\text{finite-batch noise}}
\;\cdot\; \underbrace{\kappa}_{\text{global smoothness ratio}}.
}\]
  \UPDATE{Within the global lower-bound model,} discrepancy between $\Rzer$ \UPDATE{and $\RLB$ is attributable to the relative noise terms, the global curvature ratio, or both. Discrepancy with realised optimizer outcomes can additionally arise from bound slack and mechanisms outside the model.
}

\subsection{\UPDATE{Controlled small-scale validation: input-rank intervention}}
\label{app:exp4_kappa}

  \UPDATE{We use a controlled } toy NanoGPT ($n_\text{layer}=4$, $n_\text{embd}=128$, vocab $256$, GELU activation)   \UPDATE{to test whether retaining measured directional curvature improves on the curvature-free proxy. The input-rank intervention changes both gradient structure and directional curvature, so this experiment is a controlled curvature sweep but not a proof that $\kappa$ varies in isolation.
}

\paragraph{  \UPDATE{Directional-curvature intervention}.}
We attach a forward pre-hook to a single layer (\texttt{block0.\allowbreak mlp.\allowbreak c\_fc}) that left-multiplies its input by a fixed orthogonal projection $P_r = Q_r Q_r^\top$ of rank $r$, where $Q_r \in \R^{n_\text{embd}\times r}$ has orthonormal columns drawn at initialisation.
\UPDATE{The hook leaves the layer's weight tensor and parameterization unchanged while constraining the effective input subspace to dimension $r$.}
Sweeping $r \in \{2, 4, 8, 16, 32, 64, 128\}$ (with $r=128$ the unconstrained control) moves the \UPDATE{measured directional-curvature ratio} $\hat\kappa_{\mathrm{dir}}$ over a $\sim\!5\times$ range without modifying the optimizer, the loss, or the activation function.

\paragraph{Per-trial measurement.}
At each measurement point $(r, t, \text{seed})$ and on a fresh probe batch $B$ we compute, all at the same $W$:
(i) $\Rzer = (\nucnorm{G}/\onenorm{G})^2$ from the layer-block gradient $G$;
(ii) two unit-norm descent directions $D_\text{spec} = -\polar(G)$ ($\norm{D_\text{spec}}_\op = 1$) and $D_\text{sign} = -\sign(G)$ ($\norm{D_\text{sign}}_\infty = 1$);
(iii) directional curvatures $C_\text{spec} = \langle D_\text{spec}, H D_\text{spec}\rangle$ and $C_\text{sign} = \langle D_\text{sign}, H D_\text{sign}\rangle$ via single Hessian-vector products through the layer block;
(iv) the \UPDATE{directional curvature ratio $\hat\kappa_\mathrm{dir}:= C_\text{sign}/C_\text{spec}$ and deterministic quadratic score $R_\mathrm{quad}:= \Rzer\cdot\hat\kappa_\mathrm{dir}$;
(v) the directional quadratic-model } step sizes $\eta_\text{spec} = \nucnorm{G}/C_\text{spec}$ and $\eta_\text{sign} = \onenorm{G}/C_\text{sign}$ \UPDATE{when $C_\text{spec},C_\text{sign}>0$};
(vi) the empirical one-step decreases $\Delta_\text{spec} = f(W;B) - f(W + \eta_\text{spec} D_\text{spec};B)$ and $\Delta_\text{sign}$ by the analogous step on the same batch with all other parameters held fixed.
\UPDATE{Here $R_\mathrm{quad}$ is the deterministic same-batch analogue of Eq.~\eqref{eq:empirical_index}, whereas $\Rhat$ uses stochastic averaged probes during training. When $\Delta_\text{spec},\Delta_\text{sign}>0$, we define $R_\mathrm{emp}:= \Delta_\text{spec}/\Delta_\text{sign}$ and report log-ratios. All winner statements below are restricted to these valid, same-batch, geometry-specific probe steps; outside this regime the ratio is left undefined.}

\paragraph{\UPDATE{Empirical directional-curvature sweep.}}
Table~\ref{tab:exp4_kappa} reports geometric means over $24$ trials per cell ($3$ seeds $\times\,8$ probe batches) at training step $200$. The curvature-free proxy $\Rzer$ \UPDATE{predicts a sign-side win} at every $r$\UPDATE{, while the spectral probe step has the larger positive same-batch decrease at every reported point. The deterministic score $R_\mathrm{quad}=\Rzer\hat\kappa_\mathrm{dir}$ tracks the empirical one-step ratio much more closely; the same qualitative} pattern holds at the other measurement steps and across all three seeds (Fig.~\ref{fig:exp4_kappa_rank_sweep}). \UPDATE{This is empirical calibration in this controlled setting, not an error bound that transfers to the LM/VLM gates.
}

\begin{table}[ht]
\centering
\small
\caption{Toy \UPDATE{directional-curvature sweep} at \texttt{block0.\allowbreak mlp.\allowbreak c\_fc}, step $200$, GELU, $24$ trials per row (geometric means). \UPDATE{$R_\mathrm{quad}/R_\mathrm{emp}\in[1.01,1.31]$} across the \UPDATE{aggregate rows}; $\Rzer$ disagrees in sign at every $r$.}
\label{tab:exp4_kappa}
\begin{tabular}{@{}rrrrr@{}}
\toprule
$r$ & \UPDATE{$\hat\kappa_\mathrm{dir}$} & $\Rzer$ & \UPDATE{$R_\mathrm{quad}$} & $R_\mathrm{emp}$ \\
\midrule
$2$    & $20{,}100$ & $5.7\!\cdot\!10^{-5}$ & $1.14$ & $1.13$ \\
$4$    & $14{,}100$ & $9.6\!\cdot\!10^{-5}$ & $1.36$ & $1.33$ \\
$8$    & $13{,}700$ & $1.5\!\cdot\!10^{-4}$ & $2.04$ & $2.03$ \\
$16$   & $\phantom{1}8{,}100$  & $2.9\!\cdot\!10^{-4}$ & $2.31$ & $2.24$ \\
$32$   & $\phantom{1}7{,}000$  & $4.7\!\cdot\!10^{-4}$ & $3.27$ & $3.09$ \\
$64$   & $\phantom{1}5{,}200$  & $7.9\!\cdot\!10^{-4}$ & $4.09$ & $3.54$ \\
$128$  & $\phantom{1}4{,}200$  & $1.2\!\cdot\!10^{-3}$ & $4.92$ & $3.77$ \\
\bottomrule
\end{tabular}
\end{table}

\begin{figure}[htp]
\centering
\includegraphics[width=0.8\linewidth]{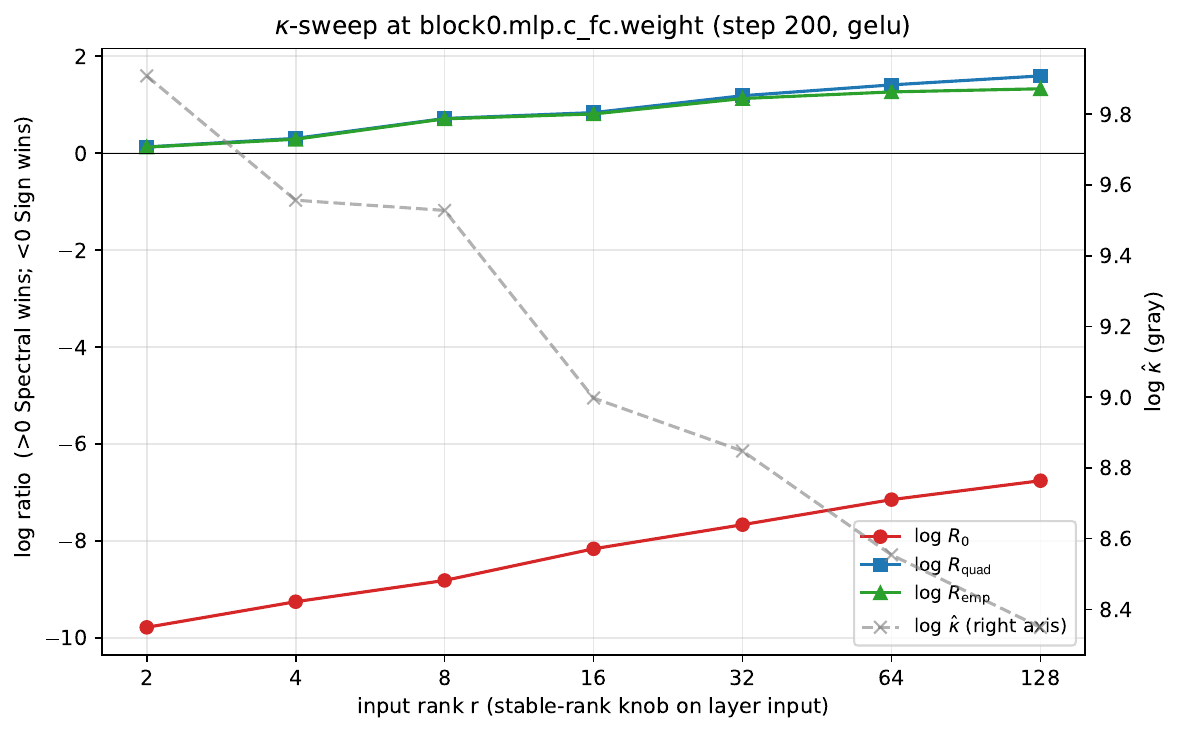}
\caption{\UPDATE{\textbf{Toy directional-curvature sweep at \texttt{block0.\allowbreak mlp.\allowbreak c\_fc} (step $200$, GELU, $24$ trials per point).}} As input rank $r$ varies \UPDATE{$2\to128$}, the directional curvature ratio \UPDATE{$\hat\kappa_\mathrm{dir}$} (gray, right axis) moves over $\sim\!5\times$. The \UPDATE{deterministic score $R_\mathrm{quad}$} (blue) and empirical one-step ratio $R_\mathrm{emp}$ (green) move together. The curvature-free proxy $\Rzer$ (red) predicts \UPDATE{the wrong side} throughout. The \UPDATE{intervention also changes $\Rzer$, so the figure demonstrates the empirical value of retaining directional curvature rather than varying curvature alone}.}
\label{fig:exp4_kappa_rank_sweep}
\end{figure}

Across $504$ controlled trials\UPDATE{, retaining the measured directional-curvature factor yields the reported qualitative agreement, whereas the curvature-free proxy selects the wrong side at every reported point. The range $[1.01,1.31]$ refers to the aggregate rows in Table~\ref{tab:exp4_kappa}, not a universal or per-trial error guarantee}.

\section{Activation-Dependent Stable Rank Theory}
\label{app:activation_theory}

This section   \UPDATE{states precisely what the } stable-rank framework of \citet{DavisDrusvyatskiy2025SpecGD}  \UPDATE{does and does not imply for our sign--spectral comparison. Their result directly concerns feature-based spectral--Frobenius quadratic majorization; we use it as architectural motivation, not as a theorem determining $L_\infty/L_\op$}.

\begin{table}[ht]
\centering
\caption{  \UPDATE{Unit-scale activation moment ratios. For non-homogeneous SiLU } and   \UPDATE{GELU these values are descriptive and are not the scale-uniform constants required by the full theorem}.}
\label{tab:nsr}
\small
  \begin{tabular}{lcc}
 \toprule
\textbf{Activation} $\sigma(t)$ &  \UPDATE{Unit-scale } $\ssigma$ &   \UPDATE{Scope } \\
\midrule
ReLU $\max(0,t)$ &   \UPDATE{$\pi\approx3.14$ } &   \UPDATE{homogeneous corollary } \\
SiLU $t(1+e^{-t})^{-1}$ &   \UPDATE{$\approx8.33$ } &   \UPDATE{descriptive; non-homogeneous } \\
GELU $t\,\Phi(t)$ &   \UPDATE{$\approx5.34$ } &   \UPDATE{descriptive; non-homogeneous } \\
ReLU$^2$ $(\max(0,t))^2$ &   \UPDATE{$6$ } &   \UPDATE{homogeneous corollary } \\
\bottomrule
\end{tabular}
 \end{table}

\subsection{ Stable Rank   \UPDATE{in a Shared Spectral--Frobenius Majorizer} }

Consider a layered model where layer $\ell$ computes $X_\ell = W_\ell A_{\ell-1}$ with post-activation matrix $A_{\ell-1} = \sigma(X_{\ell-1})$.
The   \UPDATE{shared quadratic upper model of \citet[Sec.~7]{DavisDrusvyatskiy2025SpecGD} uses two admissible feature-based coefficients:
}\begin{align}
\UPDATE{\widetilde L_\ell^{\F} }&\UPDATE{:= C_F(W) \norm{A_{\ell-1}}_{\op}^2, }&
\UPDATE{\widetilde L_\ell^{\op} }&\UPDATE{:= C_F(W) \norm{A_{\ell-1}}_{\F}^2,
}\end{align}
where $C_F(W)$ is a feature-based curvature   \UPDATE{coefficient. Their ratio is
}\begin{equation}
\UPDATE{\frac{\widetilde L_\ell^{\op}}{\widetilde L_\ell^{\F}} = \frac{\norm{A_{\ell-1}}_\F^2}{\norm{A_{\ell-1}}_\op^2} = \strank{A_{\ell-1}}.
}\end{equation}
  \UPDATE{These are feature-based majorization coefficients, not generally the least global norm-smoothness constants of the full network; the full bound also contains a parameter-only correction and a Taylor remainder}.

  \UPDATE{D\&D's nuclear-rank/stable-rank condition therefore directly motivates a spectral--Frobenius comparison. It does not determine $L_\infty$, and hence does not by itself determine our sign--spectral ratio $\kappa=L_\infty/L_\op$. In particular, the coefficient identity above cannot be substituted into a ratio of the least global norm-smoothness constants without an additional theorem}.

\subsection{Activation Noise-to-Signal Ratio}

  \UPDATE{\citet[Cor.~2.7]{DavisDrusvyatskiy2025SpecGD} give a homogeneous specialization assuming a fixed feature matrix $X$ with nonzero columns, iid Gaussian $W$ with the stated scaling, $\sigma(0)=0$, homogeneity, polynomial growth, and a nonzero Gaussian mean. Under its finite-width probability conditions, the stable rank of $A=\sigma(WX)$ is bounded by the corresponding $(1\pm\varepsilon)$ factor times
}\begin{equation}
\UPDATE{\ssigma:= \frac{\E[\sigma(\gamma)^2]}{(\E[\sigma(\gamma)])^2}.
}\end{equation}
The   \UPDATE{full Theorem~2.6 handles non-homogeneous activations using a supremum of this moment ratio over an input-scale interval, together with scale-wise nonvanishing-mean and growth assumptions. Unit-scale moments for SiLU and GELU are therefore not substitutes for that scale-uniform quantity.
}

\subsection{Explicit \texorpdfstring{$\ssigma$}{activation NSR} Computations for VLM Activations}

\paragraph{ReLU $\sigma(t) = \max(0,t)$.}
$(1,1)$-homogeneous. $\E[\sigma(\gamma)] = 1/\sqrt{2\pi}$, $\E[\sigma(\gamma)^2] = 1/2$.
\[
\ssigma = \frac{1/2}{1/(2\pi)} = \pi \approx 3.14.
\]

\paragraph{Squared ReLU $\sigma(t) = (\max(0,t))^2$.}
$(2,2)$-homogeneous. $\E[\sigma(\gamma)] = \E[\gamma^2 \cdot \mathbb{1}[\gamma > 0]] = 1/2$, $\E[\sigma(\gamma)^2] = \E[\gamma^4 \cdot \mathbb{1}[\gamma > 0]] = 3/2$ (half of the fourth Gaussian moment).
\[
\ssigma = \frac{3/2}{(1/2)^2} = 6.
\]
This is approximately   \UPDATE{twice the ReLU unit-scale moment ratio; it does not by itself predict the sign--spectral curvature ratio}.

\paragraph{SiLU $\sigma(t) = t \cdot (1+e^{-t})^{-1}$.}
Non-homogeneous, so   \UPDATE{the ratio depends on input scale}.  \UPDATE{Numerical computation gives $\ssigma\approx8.33$ at unit scale. The Hermite relation is $s^{-1}\widehat{\sigma}_1(s)=\E[\sigma'(s\gamma)]=1/2$, not $\widehat{\sigma}_1(s)=1/2$ for all $s$.
}

\paragraph{GELU $\sigma(t) = t \cdot \Phi(t)$.}
Non-homogeneous. Numerical computation gives   \UPDATE{$\ssigma\approx5.34$ at unit scale. As with SiLU, this is descriptive and does not verify the scale-uniform assumptions of the full theorem}.

\subsection{The Pretrained-Weight Regime}

The   \UPDATE{cited random-feature bounds do not apply directly to pretrained weights or to arbitrary changes of activation. Consequently, observed failures or recoveries after changing an activation cannot be attributed to stable rank from this theorem alone; weight structure, gradient structure, and local curvature all change together.
}

 \subsection{Empirical Stable Rank Across Activations}
Across our LM MLP measurements, the post-activation stable rank   \UPDATE{remains numerically small } for SiLU, ReLU, and ReLU$^2$ under both AdamW and Muon (typical values $2$--$4$).  \UPDATE{These measurements are descriptive: they do not identify the global sign--spectral curvature ratio or isolate a causal activation effect.
}

\section{Experimental Details}
\label{app:experiments}

\UPDATE{This section collects reproducibility details for GPT-2/OpenWebText pre-training, the controlled input-rank intervention, and ViT-B/16 trained from scratch on CIFAR-100. It also records the protocol correction required before a weight-tied VLM comparison can enter the empirical claims.}

\subsection{Common Diagnostic Implementation}
\label{app:diagnostic_details}

  \UPDATE{The gate computes the directional score in Eq.~\eqref{eq:empirical_index}, not the ratio $B_{\mathrm{spec}}/B_{\mathrm{sign}}$ based on global smoothness constants. A signal estimate $\widehat G_K$ is formed from $K$ minibatch gradients. The ideal natural-norm directions for a fresh gradient probe are $D_{\mathrm{sign}}=\sign(g_b)$, with $\norm{D_{\mathrm{sign}}}_\infty=1$, and $D_{\mathrm{spec}}=\polar(g_b)$, with $\norm{D_{\mathrm{spec}}}_\op=1$. Separate probe batches provide alignments with $\widehat G_K$ and autodiff} Hessian--vector \UPDATE{products. The implementation averages alignments and curvatures separately before forming $(\widehat s_g)_+^2/\widehat C_g$; finite-sample ratios are therefore Monte Carlo scores rather than unbiased estimators. An HVP is exact only for the selected finite batch and numerical graph, not for the population Hessian. A zero positive-part sign alignment makes the ratio undefined, as in Eq.~\eqref{eq:empirical_index}. Nonpositive or nonfinite curvature and undefined or nonfinite scores require implementation safeguards rather than a theorem-backed decision. The exact polar factor above describes the ideal direction; Appendix~\ref{app:scheduled_calibration} uses an FP32 five-step Newton--Schulz approximation, whose operator norm need not be exactly one. Frobenius rescaling of a direction leaves $s(D)^2/C(D)$ invariant only when the same scaling is applied to both terms; raw directional-curvature ratios are reported in the natural sign/spectral normalization}.

\subsection{Code Availability}
\label{app:code_availability}
\label{app:code}
  \UPDATE{The anonymous URL } \url{https://anonymous.4open.science/r/g-scion-4EB6/README.md}   \UPDATE{contained a README placeholder at submission. It did not contain the full training, evaluation, and diagnostic implementation needed for code-level reproduction; we acknowledge this artifact limitation explicitly.
}

\noindent\TODO{Replace the submission-time placeholder with venue-appropriate complete LM and VLM repositories, pinned release commits, environment files, and reproduction commands.}

\subsection{Language Model Pre-Training}
\label{app:lm_empirical}

\subsubsection{G-Scion Configuration Details}
\label{app:gscion}

Generalized-Scion (G-Scion) extends the Scion baseline by making the optimizer assignment of the embedding and output head layers \emph{adaptive} during training.
The optimizer assignment is denoted as (Embedding, Hidden, Output):

\begin{itemize}[leftmargin=1.2em]
\item \textbf{G-Scion (embedding):} \UPDATE{Embedding is forced spectral through warmup and is then eligible for one spectral-to-sign decision. Hidden remains Muon and output remains SignSGD.}
\item \textbf{G-Scion (output):} \UPDATE{Embedding remains SignSGD, hidden remains Muon, and output is forced spectral through warmup before becoming eligible for one spectral-to-sign decision.}
\item \textbf{G-Scion (embedding, output):} \UPDATE{Both edge groups are forced spectral through warmup and independently eligible for one spectral-to-sign decision; hidden remains Muon.}
\end{itemize}

\paragraph{Evaluation protocol.}
\UPDATE{$\Rhat$ is evaluated every $N{=}50$ steps after warmup. Each valid log score is smoothed with EMA factor $0.9$; a layer switches only after two consecutive observations beyond the $0.1$ margin around the nominal zero threshold. Fresh signal, direction, and HVP probes are aggregated as described above. Once switched, a layer remains locked.}

\noindent\TODO{Add held-out same-iterate score calibration; the $31$M static $2{\times}2$ controls, submitted forced-warm-up gate, corrected one-way gate, route replay, and bidirectional routing; and one-factor $N$, $K$, and $\tau$ sensitivity with same-allocation timing and peak memory.}

\paragraph{Baselines.}
\begin{itemize}[leftmargin=1.2em]
\item \textbf{Scion:} SignSGD (embedding), Muon (hidden), SignSGD (output).
\item \textbf{Muon:} AdamW (embedding), Muon (hidden), AdamW (output).
\end{itemize}

\paragraph{Cross-layer interaction effects.}
\UPDATE{A layerwise score is not a joint-assignment guarantee: scoring blocks separately omits cross-block terms in the quadratic model of a simultaneous update. The available records do not establish a joint-oracle failure rate or isolate a causal mechanism for such failures. Quantifying either requires matched joint candidate updates from the same checkpoints, evaluated on independent batches; final losses from separately trained routes do not identify these one-step interactions.}

\subsubsection{Batch-Size Sweep on Tiny Shakespeare}
\label{app:lm_batchsize}

\begin{figure}[ht]
\centering
\includegraphics[width=0.96\linewidth]{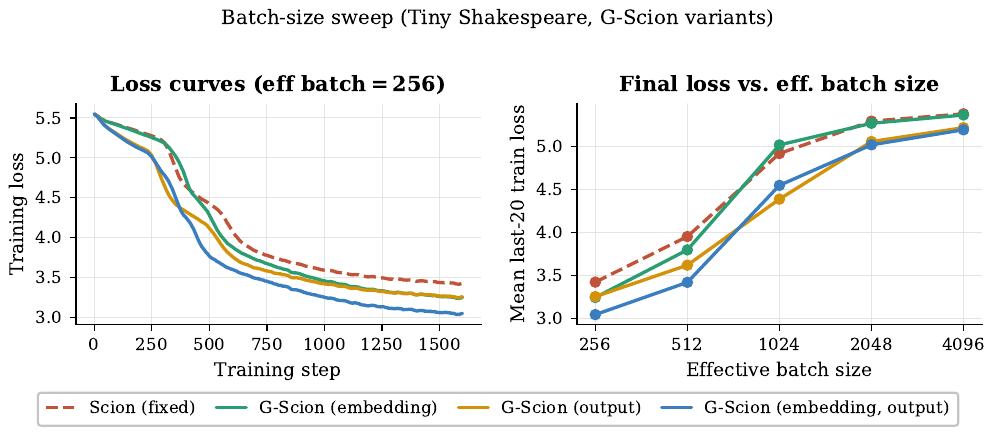}
\caption{\textbf{Batch-size sweep (Tiny Shakespeare, G-Scion variants).}
\emph{Left:} Training loss curves at effective batch $256$ ($1{,}600$ steps). \UPDATE{G-Scion (embedding, output)} leads throughout; \UPDATE{G-Scion (embedding)} tracks Scion closely, while \UPDATE{G-Scion (output)} and \UPDATE{G-Scion (embedding, output)} pull away via output-layer adaptation.
\emph{Right:} Mean last-$20$ training loss vs.\ effective batch size ($\log_2$ scale). All G-Scion variants beat Scion at every  \UPDATE{reported } batch size; the  \UPDATE{observed } margin is largest at small batch  and shrinks as the effective batch grows.}
\label{fig:batch_sweep}
\end{figure}

  \UPDATE{The batch-size sweeps observe that} G-Scion's advantage is largest at small $B$. \UPDATE{This pattern is compatible with a smaller spectral relative-noise penalty, but the theory does not predict it universally without an ordering between $\alpha$ and $\beta$. The high-SNR global-smoothness ratio contains no batch-size dependence. These results are therefore empirical observations rather than confirmation of a universal small-batch theorem}.

\subsubsection{OWT Activation Ablation with \texorpdfstring{$\Rhat$}{the directional score} (12 Runs)}
\label{app:lm_owt}

Twelve NanoGPT ($31$M, $6$ layers) experiments on OpenWebText use the   \UPDATE{directional Monte Carlo score
} $\Rhat$ (Eq.~\ref{eq:empirical_index}) for gating decisions, organised in a $3{\times}4$ grid of
MLP activations (GELU, ReLU$^2$, SiLU) $\times$ optimizer variants (Scion, G-Scion (embedding),
G-Scion (output), G-Scion (embedding, output)).
All runs: $2{,}000$ steps, effective batch $128$ (physical $8{\times}16$ microbatches),
peak learning rate $10^{-3}$ (cosine schedule, min $10^{-4}$, $100$-step warmup).

\paragraph{Validation losses.}
Table~\ref{tab:a100_owt_full} reports final validation losses for all $12$ runs. The validation-loss ordering (lower is better) G-Scion (embedding, output) $<$ G-Scion (embedding) $<$ G-Scion (output) $<$ Scion holds for all three activations. In the single-seed block, the best-to-worst spread across activations within each method is $<0.05$ in validation loss, while the cross-method spread reaches $0.37$ in the GELU row. The two-seed means have a within-method activation spread of up to $0.052$ but these descriptive comparisons do not establish a general ranking of activation and optimizer effects.

\begin{table}[t]
\centering
\caption{Final validation loss, OWT $31$M, $\Rhat$. \textbf{Top:} single-seed runs ($12$ runs, seed 1337). \textbf{Bottom:} two-seed replication ($24$ runs, seeds 42 and 2024; mean $\pm$ population std). The validation-loss ordering (lower is better) G-Scion (embedding, output) $<$ G-Scion (embedding) $<$ G-Scion (output) $<$ Scion is uniform across all activations and both seed sets.}
\label{tab:a100_owt_full}
\footnotesize
\setlength{\tabcolsep}{4pt}
\renewcommand{\arraystretch}{1.08}

\begin{tabular}{lccccc}
\toprule
\textbf{Activation} & \textbf{Scion} & \textbf{G-Scion (embed.)} & \textbf{G-Scion (out.)} & \textbf{G-Scion (embed., out.)} & \textbf{Best--Worst} \\
\midrule
\multicolumn{6}{l}{\textit{Single seed (seed 1337)}} \\
\midrule
GELU       & $5.526$ & $5.278$ & $5.459$ & $\mathbf{5.161}$ & $0.365$ \\
ReLU$^2$   & $5.490$ & $5.293$ & $5.421$ & $\mathbf{5.196}$ & $0.294$ \\
SiLU       & $5.531$ & $5.309$ & $5.462$ & $\mathbf{5.205}$ & $0.326$ \\
\midrule
Mean               & $5.516$ & $5.293$ & $5.447$ & $\mathbf{5.187}$ & \\
$\Delta$ vs.\ Scion & --- & $+0.22$ & $+0.07$ & $\mathbf{+0.33}$ & \\
\midrule
\multicolumn{6}{l}{\textit{Two-seed replication (seeds 42 \& 2024; mean $\pm$ std, $n=2$)}} \\
\midrule
GELU       & $5.531 \pm 0.011$ & $5.315 \pm 0.011$ & $5.448 \pm 0.012$ & $\mathbf{5.181 \pm 0.000}$ & $0.350$ \\
ReLU$^2$   & $5.507 \pm 0.001$ & $5.351 \pm 0.001$ & $5.431 \pm 0.001$ & $\mathbf{5.232 \pm 0.000}$ & $0.275$ \\
SiLU       & $5.559 \pm 0.015$ & $5.324 \pm 0.020$ & $5.454 \pm 0.008$ & $\mathbf{5.205 \pm 0.001}$ & $0.354$ \\
\midrule
Mean               & $5.532$ & $5.330$ & $5.444$ & $\mathbf{5.206}$ & \\
$\Delta$ vs.\ Scion & --- & $+0.20$ & $+0.09$ & $\mathbf{+0.33}$ & \\
\bottomrule
\end{tabular}
\end{table}

\begin{figure}[ht]
\centering
\includegraphics[width=0.50\linewidth]{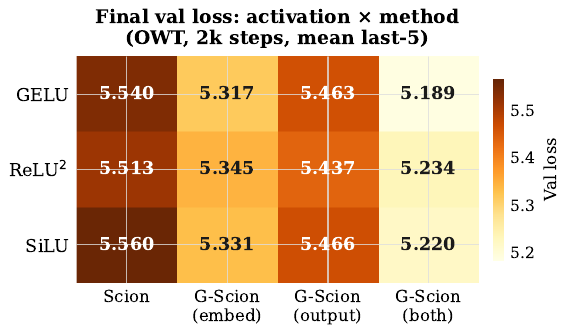}
\caption{\textbf{Activation $\times$ method validation loss} (OWT, 31M, 2\,000 steps,
$\Rhat$; mean of last-5 checkpoints). Lower numerical values indicate better validation loss.
The validation-loss ordering (lower is better) G-Scion (embedding, output) $<$ G-Scion (embedding) $<$ G-Scion (output) $<$ Scion holds uniformly across all
three activations.  Activation function is a second-order effect (within-column spread
$<0.05$ \UPDATE{in validation loss}).}
\label{fig:lm_a100_heatmap}
\end{figure}

\begin{figure}[ht]
\centering
\includegraphics[width=0.96\linewidth]{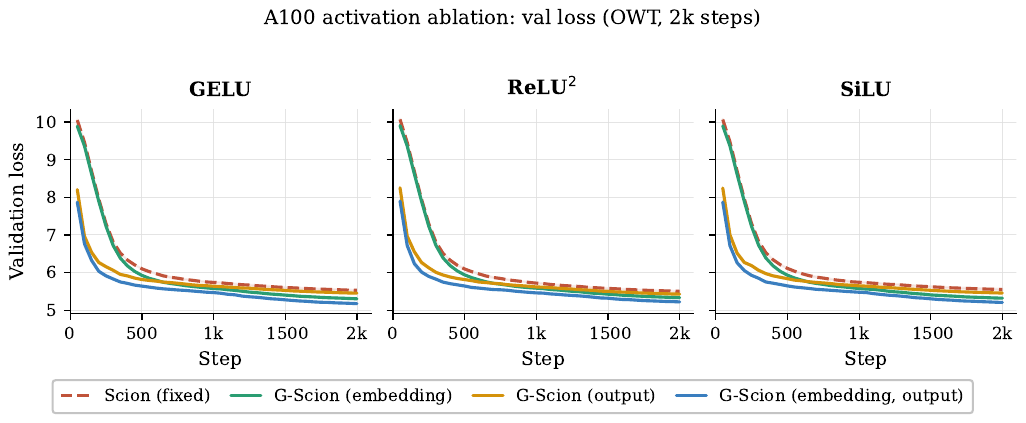}
\caption{\textbf{Validation loss curves} (OWT, 31M, 2\,000 steps, $\Rhat$)
for all three MLP activations. Colours: G-Scion (embedding, output) (blue), G-Scion (embedding) (green),
G-Scion (output) (orange), Scion (vermilion).  Ordering is preserved across activations;
G-Scion (embedding, output) leads in every panel.}
\label{fig:lm_a100_val_loss}
\end{figure}

\paragraph{Diagnostic readings and gate decisions.}
The raw diagnostic at a valid measurement $j$ is $s_{\ell,j}=\log\Rhat_{\ell,j}$. The gate instead acts on its exponential moving average, initialized by $z_{\ell,1}=s_{\ell,1}$ and subsequently updated as $z_{\ell,j}=\beta z_{\ell,j-1}+(1-\beta)s_{\ell,j}$. For the stated spectral-to-sign rule, the decision requires two consecutive smoothed readings below $\tau-m$, followed by locking (Algorithm~\ref{alg:g-scion}). A fraction of raw readings below zero is therefore not a switch frequency or a confirmation rate. Shadow measurements on a fixed route likewise do not constitute routing decisions.

Figure~\ref{fig:lm_logA} is retained as a descriptive diagnostic trace. We omit the historical pooled-score table and the exact first-decision range because their raw-versus-smoothed definitions, sampled steps and aggregation weights cannot be reconstructed from the available summaries. In particular, a summary over the training horizon must not be labelled as exclusively pre-switch. Reproducible numerical summaries must identify the runs and activations, raw score or EMA, included steps and phase boundaries, valid-observation count, and weighting across runs. The reported route change at step $150$ is stated separately below; no pooled sign-side percentage is used to certify it.

\begin{table}[ht]
\centering
\caption{AdamW-side reference at $31$M (same setting as \tabref{tab:gscion}). \emph{Muon}: AdamW on embed/head + Muon on hidden, the standard practical recipe. \emph{G-Muon}: one-way AdamW$\to$spectral gate on embed/head + Muon on hidden. AdamW adds momentum and per-parameter adaptivity outside the one-step bound, so this is a reference rather than a like-for-like sign-vs-spectral comparison; the gate does not improve over Muon in this regime. The scope of the theorem does not explain this performance difference.}
\label{tab:gscion_muon}
\small
\begin{tabular}{lcc}
\toprule
\textbf{Method} & \textbf{Final val loss} & \textbf{$\Delta$ vs.\ Muon} \\
\midrule
Muon (AdamW embed/head + Muon hidden)             & $4.755 \pm 0.035$ & --- \\
G-Muon (AdamW$\to$Spectral embed/head + Muon hidden) & $4.851 \pm 0.016$ & $-0.096$ \\
\bottomrule
\end{tabular}
\end{table}

\paragraph{Output gate timing.}

For G-Scion (output) and G-Scion (embedding, output), the output gate fires at step $150$ universally across all $6$
runs (two methods $\times$ three activations).
The stated schedule begins decisions at $t=100$, with cadence $N=50$; a switch at $t=150$ is the second scheduled opportunity. This timing is distinct from the numerical value of either the raw score or its EMA. The exact first-decision score range is not used as quantitative evidence until the original logging and aggregation records are reconciled (see the diagnostic-reporting paragraph above).
After locking to sign at step $150$, the output stays in sign for the remaining
${\approx}1{,}850$ steps.

\paragraph{Raw and normalized proxy comparisons.}\label{par:Rzer_comparison}
\UPDATE{For nonzero $G\in\mathbb{R}^{m\times n}$, the raw proxy $\Rzer=\nucnorm{G}^2/\onenorm{G}^2\le1$ cannot favor spectral at a unit threshold. The dimension-normalized score $I_{\mathrm{proxy}}=\sqrt{mn\,\Rzer}$ can exceed one; its unit-threshold decision is equivalent to substituting $\kappa=mn$, not to measuring curvature. The historical proxy-run summaries do not establish the exact normalization and gate configuration, so we do not attribute their reported spectral-side output readings to raw $\Rzer$ or use the claimed matched loss comparison as evidence. The controlled same-iterate experiment in Appendix~\ref{app:exp4_kappa} instead compares $\Rzer$ with $R_\mathrm{quad}=\Rzer\hat\kappa_\mathrm{dir}$ using measured directional curvature. This comparison does not identify the global ratio $\kappa$ or equate $\Rhat$ with $\Rinft$.}

\paragraph{Replication across seeds.}
The identical $3\times 4$ grid was replicated with two additional seeds ($42$ and $2024$), for $24$ runs in total; per-cell means are in the lower block of Table~\ref{tab:a100_owt_full}. The validation-loss ordering (lower is better) G-Scion (embedding, output) $<$ G-Scion (embedding) $<$ G-Scion (output) $<$ Scion holds for all three activations in the single-seed results and the two-seed means. From the displayed rounded values, the per-cell improvement over Scion differs between these summaries by up to $0.041$: for ReLU$^2$ with G-Scion (embedding), the improvement is $5.490-5.293=0.197$ for seed $1337$, versus $5.507-5.351=0.156$ for the two-seed mean. Raw validation losses differ by up to $0.058$ per cell. Thus the method ordering is reproduced in these summaries, but agreement within $0.02$ per cell is not supported.

\paragraph{Learning-rate stratification at the output head.}
\UPDATE{A sweep over four learning rates and five effective batch sizes ($80$ runs) shows that the phase-aggregated output score is learning-rate dependent (Table~\ref{tab:lr_output_logA}). At conservative learning rates both phase means are negative; at $3\times10^{-3}$ the spectral-phase mean is negative and the subsequent sign-phase mean is positive. This is a phase-stratified aggregate, not a single-run bidirectional transition trace.}

\begin{table}[ht]
\centering
\caption{Output $\log\Rhat$ by learning rate and phase (OWT $31$M, 80-run sweep,
$\Rhat$). ``Spectral phase'': steps $0$--$100$ (before gate). ``Sign phase'': steps $150$--$1950$
(after gate). All gated-output runs switch at step $150$ regardless of LR.}
\label{tab:lr_output_logA}
\small
\begin{tabular}{lcc}
\toprule
\textbf{LR} & \textbf{Out $\log\Rhat$, spectral phase} & \textbf{Out $\log\Rhat$, sign phase} \\
\midrule
$10^{-4}$      & $-1.83$ \;($100\%$ sign-side) & $-1.76$ \;($100\%$ sign-side) \\
$3{\times}10^{-4}$ & $-2.46$ \;($100\%$ sign-side) & $-1.94$ \;($100\%$ sign-side) \\
$10^{-3}$      & $-1.12$ \;($100\%$ sign-side) & $-0.29$ \;($72\%$ sign-side)  \\
$3{\times}10^{-3}$ & $-0.27$ \;($100\%$ sign-side) & $\mathbf{+0.70}$ \;($19\%$ sign-side) \\
\bottomrule
\end{tabular}
\end{table}

\subsubsection{\UPDATE{Scale-Sweep Provenance and Scope}}
\label{app:lm_scale_sweep}

\UPDATE{The original scale analysis combined two different diagnostics. Its $1$B plot displayed a curvature-free proxy whose exact normalization is not established by the available summaries, whereas the available HVP-based output score is negative and the realised route becomes sign after warmup. Several $163$M/$406$M method--seed cells lack complete directional records. We therefore exclude those tables and figures from score validation. A matched scale test records, for every seed and method, the directional score, parameter route, and static sign/spectral edge controls.}

\noindent\TODO{After the $31$M score-validity and assignment tests pass, add the matched $31$M/$1$B directional scores, realised routes, and static output-head ordering. Claim a scale reversal only if both the score sign and the static ordering reverse.}
\subsubsection{Coordinate invariance}
\label{app:split_invariance}

For a gradient matrix $G$ and orthogonal matrices $U,V$, the transformed gradient $G'=UGV^\top$ has the same spectral and nuclear norms as $G$. Thus the spectral geometry is invariant under orthogonal changes of coordinates. This statement concerns a consistent reparameterization of the objective and its gradients, not an isolated transformation of the weight matrix.

The entrywise $\ell_\infty$ geometry underlying SignSGD is generally not invariant under arbitrary orthogonal changes of coordinates~\citep{balles2020geometry}. It is, however, invariant under coordinate permutations and sign flips. Consequently, attention-head or neuron reorderings that only permute coordinates do not distinguish sign from spectral geometry on invariance grounds.

\subsection{Vision: ViT-B/16 on CIFAR-100}
\label{app:vit_cifar}

\UPDATE{This appendix reports image-classification experiments that apply the four edge-assignment trajectories of \secref{sec:lm_validation} to ViT-B/16~\citep{dosovitskiy2020image} trained from scratch on CIFAR-100~\citep{krizhevsky2009learning}. The evidence concerns final assignment outcomes rather than validation of the directional score.}

\subsubsection{Setup}
\label{app:vit_setup}

\paragraph{Architecture and data.}
ViT-B/16 ($12$ encoder blocks, width $768$, $86$M parameters) with \texttt{num\_classes}${=}100$, image size $224$, trained from random initialisation on CIFAR-100 ($50$k train, $10$k test; standard random-crop+flip augmentation).   \UPDATE{Random initialization alone does not verify all assumptions } of \Propref{prop:strank_activation}\UPDATE{; the experiment is an empirical cross-modal check rather than an application of that theorem}.

\paragraph{Optimizer layer partition.}
\UPDATE{Hidden attention/MLP layers always use Muon. The edge groups are $\{\text{conv\_proj}+\text{class\_token}+\text{pos\_embed},\,\text{heads.head}\}$, matching the LM ``embedding'' and ``output'' roles. The runs use a $100$-step forced warmup and one-way locking. The HVP score is nonfinite at many ViT patch-embedding measurements, so neither the implementation flag nor the plotted proxy is used to infer a geometry preference.}

\paragraph{Sweep grid.}
We sweep two learning rates $\{3\times 10^{-4},\,3\times 10^{-3}\}$ and three effective batch sizes $\{128, 256, 512\}$, with two seeds ($1337$ and $3407$) per cell, for $48$ runs at $1{,}000$ steps each. A $3$-seed replicate at effective batch $256$ and learning rate $3\times 10^{-3}$ ($12$ further runs) corroborates the corresponding sweep cell.

\subsubsection{Validation Accuracy and Batch-Size Sensitivity}
\label{app:vit_headline}

\begin{table}[ht]
\centering
\caption{\UPDATE{\textbf{Final validation accuracy on ViT-B/16 CIFAR-100} ($1{,}000$ steps, $2$ seeds per cell, lr${=}3{\times}10^{-3}$). The joint spectral-edge assignment trajectory exceeds fixed Scion at every reported effective batch size. $\Delta$ is its accuracy minus Scion accuracy in percentage points.}}
\label{tab:vit_cifar_batch}
\small
\resizebox{\linewidth}{!}{
\begin{tabular}{lccccc}
 \toprule
\textbf{Effective batch} & \textbf{Scion} & \textbf{G-Scion (embedding)} & \textbf{G-Scion (output)} & \textbf{G-Scion (embedding, output)} & $\Delta$ \\
\midrule
$128$ & $30.9 \pm 1.2$ & $32.5 \pm 2.1$ & $43.8 \pm 2.8$ & $\mathbf{42.8 \pm 0.4}$ & $+11.9$ \\
$256$ & $39.5 \pm 0.7$ & $41.0 \pm 0.4$ & $45.3 \pm 0.4$ & $\mathbf{51.0 \pm 1.1}$ & $+11.5$ \\
$512$ & $40.0 \pm 0.7$ & $43.9 \pm 2.7$ & $44.1 \pm 1.6$ & $\mathbf{47.9 \pm 1.2}$ & $+\phantom{0}7.9$ \\
\bottomrule
\end{tabular}
}
\end{table}

\UPDATE{The joint edge-assignment trajectory exceeds fixed Scion at every reported effective batch size (Table~\ref{tab:vit_cifar_batch}). At batch $128$, the output-only and joint cells overlap within seed spread; at batches $256$ and $512$, the joint cell is highest. The observed gap shrinks with batch size, but this sweep does not isolate the relative noise terms required for a theoretical direction-of-effect claim.}

\subsubsection{\UPDATE{Score Provenance}}
\label{app:vit_Is}

\UPDATE{The historical ViT plot was labelled $\Rhat$, but the patch-embedding directional score is valid at only three of eight recorded measurements and becomes nonfinite. The plot is therefore removed, and the ViT result is used only as evidence that the alternative assignment trajectories improve final accuracy. A matched static spectral-edge control is required to attribute this gain to online identification.}

\noindent\TODO{Diagnose the nonfinite ViT patch-embedding score by component, direction, and sample. After a finite preflight, add matched static spectral-edge and adaptive controls or retain the outcome-only scope.}

\subsubsection{Per-layer decomposition: nonadditive outcomes}
\label{app:vit_decomposition}

\UPDATE{At effective batch $256$, the joint gain over Scion is larger than the sum of the two single-group gains, whereas the $31$M decomposition in Table~\ref{tab:gscion}(b) is close to additive. This is nonadditivity of final training outcomes; it does not isolate cross-block curvature or establish an error in a layerwise score. Matched joint updates at a common checkpoint are needed to test those interactions.}

\subsection{\UPDATE{Weight-Tied VLM Protocol}}
\label{app:vlm_details}

\subsubsection{\UPDATE{Architecture and estimand}}
\label{app:vlm_setup}

\UPDATE{nanoVLM~\citep{nanovlm2025} combines a SigLIP2 vision encoder~\citep{tschannen2025siglip2}, an MLP projector, and a SmolLM2 decoder~\citep{allal2025smollm2}. The decoder's token embedding and output head are aliases of the same $\mathrm{Parameter}$. The relevant estimand is therefore the geometry assigned to one shared embedding--head tensor, including gradient contributions from both roles, not separate ``embedding'' and ``head'' decisions.}

\subsubsection{\UPDATE{Audit of the historical configurations}}
\label{app:vlm_partition}

\UPDATE{The historical runs set the full directional diagnostic off. During optimizer construction, the shared tensor was first encountered through its embedding alias and assigned once to the AdamW group. Run labels referring to a head route therefore do not identify the realised optimizer on that tensor. The old proxy values and downstream metrics cannot support a sign-versus-spectral head comparison, a VLM head reversal, or activation causality, and are not included as empirical evidence.}

\subsubsection{\UPDATE{Required matched design}}
\label{app:vlm_rhat}

\UPDATE{A matched evaluation must (i) enumerate all aliases and verify a single parameter identity, (ii) assign exactly one realised optimizer route to the shared tensor, (iii) compute its joint alignment and HVP, (iv) compare static sign and static spectral routes at matched hyperparameters, and (v) evaluate the adaptive route against those static corners. The projector should be crossed independently because it is the interface receiving visual features. Untying the embedding and head changes the architecture and is therefore a secondary control rather than a substitute for the tied comparison.}

\noindent\TODO{Add the corrected tied-VLM audit: alias identity, one shared optimizer occurrence, joint alignment/HVP finite rates, invalid-score abstentions, realised routes, and state transfer. Document five-step Newton--Schulz as the deployed spectral direction and exact polar only as the failure-isolated P0 reference.}

\subsubsection{Sweep grid and seeds}
\label{app:vlm_sweep}

\UPDATE{The corrected tied comparison and its static controls are required before reporting VLM performance or a cross-architecture reversal. No result from that planned comparison is claimed in this version.}

\noindent\TODO{Replace this placeholder with the tied shared-tensor/projector $2{\times}2$ deployed Signum/Muon grid, adaptive condition, and practical AdamW-edge reference across paired seeds. Report an untied role control only as a changed-architecture study, and add an activation-only intervention only if a causal activation claim is retained.}

\section{Comprehensive Related Work}
 \label{app:related_work}

 \paragraph{Spectral and matrix geometry in large-model training.}
Many widely used optimizers for deep networks can be viewed as preconditioned first-order methods, including coordinatewise adaptive methods such as AdaGrad and Adam and matrix-preconditioned methods such as K-FAC and Shampoo \citep{duchi2011adaptive,kingma2015adam,martens2015optimizing,gupta2018shampoo}. A complementary perspective emphasises the geometry of the update for matrix-valued parameters \citep{carlson2015stochastic,Ber24,Vep25}, recently exemplified by Muon \citep{jordan2024muon}, which performs a spectral-norm steepest-descent-style update and has emerged as a competitive alternative to AdamW in language-model pretraining   \UPDATE{\citep{Liu25b,zeng2025glm,Sha25}}; convergence bounds and a critical-batch-size analysis covering Muon with Nesterov momentum and weight decay are given by \citet{sato2025muonconvergence}. Muon has been shown to scale to billion-parameter LLMs and MoE training   \UPDATE{\citep{Liu25b}}, and related matrix-preconditioned methods such as SOAP also improve over AdamW in large-batch pretraining regimes \citep{Vya24,qiu2026hyperparameter}. The remaining open question is which geometry should be used at which layer of a given architecture, and under what conditions its advantage persists.

\paragraph{Mixed recipes and the assignment problem.}
The practical answer so far has been to mix geometries. Successful Muon-style recipes typically apply matrix-aware updates to hidden weight matrices while retaining AdamW-like updates on embeddings, normalization parameters, biases, or output layers   \UPDATE{\citep{Liu25b,Esc26}}. Hybrid methods make this selectivity even more explicit. SOAP can be read as Adam in a rotated basis induced by a matrix preconditioner \citep{Vya24}, COSMOS splits optimization across eigensubspaces using SOAP on leading directions and Muon on the remainder \citep{liu2026cosmos}, and adaptive Muon variants reintroduce variance adaptation rather than treating orthogonalization as sufficient \citep{Fra25,Si25,Li25e}. This line of work strongly suggests that the field is converging on mixed geometry rather than on a single universally best optimizer. Yet the assignment itself remains largely heuristic: the usual partition into ``matrix'' and ``non-matrix'' parameters is inherited from practice and often must be re-validated across scale, architecture, and tuning budget \citep{wen2026fantastic,qiu2026hyperparameter,Sem25}.

\paragraph{Why AdamW-like and sign-based geometry persists.}
Selectivity in mixed-optimizer recipes is consistent with transformer heterogeneity. Hessian-based analyses show that transformer blocks have markedly different local curvature, so a single Euclidean learning-rate geometry is a poor fit across the whole network \citep{zhang2024why}. Empirical work finds that Adam-like adaptivity is disproportionately important on a small set of sensitive parameter groups, especially the last layer and LayerNorm parameters, while other blocks tolerate simpler updates \citep{zhao2025deconstructing}. Sharpness and gradient diagnostics reach a similar conclusion: blockwise sharpness disparity persists throughout training \citep{wang2025the}, and gradient heterogeneity helps explain why coordinatewise or sign-like normalization remains effective where SGD fails \citep{tomihari2025understanding}. SignSGD \citep{bernstein2018signsgd} and the geometric interpretation of sign descent as entrywise steepest descent \citep{balles2020geometry} provide the sign-side foundation, and recent theory argues that Adam benefits from favourable $\ell_\infty$-geometry of the loss landscape \citep{xie2025adam}. Together these results indicate why AdamW-like behaviour remains difficult to displace entirely on the embedding and output side, even when spectral methods improve training on the internal weight matrices.

\paragraph{Criteria for spectral advantage and the role of stable rank.}
Among existing attempts to move beyond heuristic assignment, the closest precursor is the spectral-versus-GD analysis of \citet{DavisDrusvyatskiy2025SpecGD}. That work derives a layerwise condition under which spectral updates should outperform Euclidean gradient descent, governed by the gradient's nuclear-to-Frobenius ratio relative to the stable rank of the incoming activations. It is especially relevant because it ties optimizer preference to measurable properties of the current layer rather than to an architecture label alone, and because its empirical results support the view that internal transformer matrices are especially favorable to spectral geometry. The same work also links activation choice to optimizer suitability through stable rank, providing a first concrete bridge between architectural nonlinearities and geometry selection. A complementary Muon-specific account is given in \citet{wang2026muon}, which traces Muon's gains mainly to value-output attention and feedforward matrices and links that advantage to more isotropic singular spectra under heavy-tailed data. These results bring the literature close to a principled layerwise criterion, but they do not yet resolve the comparison that modern mixed recipes actually instantiate: spectral geometry versus AdamW-like coordinatewise or sign-based geometry.

\paragraph{Activation dependence and non-Euclidean noise.}
The architecture dependence of optimizer preference remains underdeveloped in the literature. Modern transformers rely on a range of activations, including GELU, Swish, and related variants \citep{hendrycks2016gelu,ramachandran2017swish,so2021primer}, but optimizer choice is usually treated as largely separate from activation design. The stable-rank perspective of \citet{DavisDrusvyatskiy2025SpecGD} begins to close that gap by tying activation-induced representation structure to the case for spectral updates. On the stochastic side, non-Euclidean gradient noise scales provide a way to compare geometries at finite batch size rather than only in deterministic limits \citep{naganuma2026adaptivebatchsizesusing}. Together, these ideas suggest that geometry preference should vary not only across layers, but also across batch regimes, activations, and architectures.

\paragraph{Unified geometry views and positioning of the present work.}
A broader theoretical literature clarifies that these comparisons should be posed in geometric terms. Norm-based and trust-region views   \UPDATE{interpret } sign-based   \UPDATE{and spectral updates through different metrics or constraint sets \citep{Ber24,pethick2025training,Kov25}. Relative to \citet{DavisDrusvyatskiy2025SpecGD}, the key shift here is from a } spectral-versus-Euclidean  \UPDATE{quadratic-model } comparison to the spectral-versus-sign   \UPDATE{assignment faced by modern mixed recipes. Relative to fixed prescriptions such as Scion \citep{pethick2025training}, our contribution is a measurable layerwise instrument: a signal--curvature factorisation derived from global lower bounds together with an empirically evaluated directional score. The controlled intervention tests the score at a common iterate, the $31$M output supplies an operational sign-side case, and the $31$M/ViT trajectories test alternative assignments against fixed Scion.}

\section{Compute Resources}
\label{app:compute}

\paragraph{Hardware.}
GPT-2 and ViT/CIFAR-100 experiments ran on NVIDIA A100 80\,GB GPUs, split between RunPod cloud instances (for the GPT-2 scale sweep, Appendix~\ref{app:lm_scale_sweep}, and the ViT-B/16 sweep, Appendix~\ref{app:vit_cifar}) and an institutional cluster (for the OWT $31$M activation ablation, Appendix~\ref{app:lm_owt}, and the small-scale \UPDATE{directional-curvature} diagnostic, Appendix~\ref{app:exp4_kappa}). Peak parallelism was four A100s, dominated by a $3$-GPU RunPod pod for the largest scales. \UPDATE{Historical nanoVLM pilots ran on a single NVIDIA H100 80\,GB per run; they are excluded from the empirical claims for the protocol reasons in Appendix~\ref{app:vlm_details}.}

\paragraph{Per-experiment wall-clock.}
The $1$B GPT-2 runs are the dominant cost at $\approx 140$ single-GPU hours per seed; the $163$M and $406$M runs scale roughly with parameter count. The OWT $31$M activation ablation completes at $\approx 1$ hour per run, the $12$-run ViT-B/16 paper sweep at $\approx 6.4$ hours per run, the $48$-run ViT learning-rate sweep at $\approx 2.8$ hours per run, and the small-scale \UPDATE{directional-curvature} probes within ten minutes per run. \UPDATE{The historical $1$h-to-$2$h nanoVLM comparison included exhaustive per-component diagnostics and does not isolate gate cost. In a separate end-to-end timing experiment with decision cadence $N{=}100$, G-Scion measured approximately $4\%$ overhead over its matched baseline. This measurement is configuration-specific; cost depends on cadence, probe count, unlocked groups, model, and implementation.}

\paragraph{Total compute budget.}
\UPDATE{The experimental program used approximately $2{,}400$ A100-GPU-hours plus $100$ H100-GPU-hours, excluding unlogged pilot runs. The A100 total includes about $1{,}700$ hours for the $1$B scale runs even though those runs are not used as directional-score evidence; the H100 total includes the historical nanoVLM studies excluded from the empirical claims.}

\clearpage
\section{Held-out scheduled-update calibration: development study}
\label{app:scheduled_calibration}
\paragraph{Design and estimands.}
We trained four fresh 31M language-model sources using static SS and PP assignments and seeds 1337 and 2024, for 2,000 updates. Here S denotes sign and P spectral, with the embedding assignment first and output assignment second. Matrix updates were memoryless; this is not a momentum-Muon study. The diagnostic grid uses checkpoints after 100, 700 and 1,999 updates, four units (joint token/position embeddings, output head, attention and MLP), gradient batches $B\in\{32,128,512\}$ and two repetitions, giving 288 paired probes. We use $K=12$, $M=8$, Hessian batch $H=128$, FP32 diagnostics and five Newton--Schulz iterations. Each sign/spectral candidate applies its actual scheduled displacement to the selected unit from the identical checkpoint; other units remain fixed. Model, optimizer and RNG restoration checks passed. Signal, direction, Hessian and update samples occupy separate training-token regions. Both candidates are evaluated on the same 32 disjoint held-out spans, aggregated into eight paired blocks.

We distinguish the implemented self-curvature score from a full-group diagnostic that also includes cross-member Hessian terms. Neither estimates $\RLB$. The alternative predictor $Q_{\mathrm{actual}}=-g^\top u-\tfrac12 u^\top H u$ uses the actual rounded scheduled displacement $u$ and independently estimated signal and curvature. It is a separate development predictor, not a retrospectively validated replacement for the gate.

\begin{table}[ht]
\centering
\small
\caption{Development scheduled-update ranking. Regret is selected held-out loss minus the better sign/spectral loss, averaged over the 288 tuples. The four undefined directional ratios retain the source-route fallback.}
\label{tab:scheduled_calibration}
\begin{tabular}{lrr}
\toprule
Selection rule & Correct / 288 & Mean regret (nats/token) \\
\midrule
Full-group directional diagnostic + fallback & 219 & 0.005037 \\
Always sign & 24 & 0.024432 \\
Always spectral & 264 & 0.003471 \\
$Q_{\mathrm{actual}}$ (development only) & 288 & 0.000000 \\
\bottomrule
\end{tabular}
\end{table}

\paragraph{Failures below and near a score of one.}
With prespecified log-score bins below $-0.2$, between $-0.2$ and $0.2$, and above $0.2$, the full-group threshold makes 7/30, 129/169 and 83/85 correct choices, respectively. The policy itself switches its preference at score one. Four ratios are undefined because of negative sign alignment in late SS output-head probes at $B=32$; the source-route fallback chooses sign and loses all four. No such cases are excluded from the primary table. Among the 284 defined pairs, self-curvature and full-group scores rank 205 and 219 correctly. All 28 changed choices occur in joint embeddings, where accuracy increases from 39/72 to 53/72. This diagnoses within-group cross curvature, not full embedding/head interactions.

\paragraph{Noise, route dependence and uncertainty.}
The full-group score makes 64/96, 81/96 and 74/96 correct choices at $B=32,128,512$, while the recorded uncorrected variance-to-squared-mean-gradient proxy decreases from 1.838 to 0.730 to 0.231. Lower measured noise therefore does not monotonically repair agreement. Always spectral has lower pooled regret, but the directional diagnostic has lower regret on each SS source; all PP-source winners are spectral. The balanced diagnostic grid is not an estimate of deployment frequencies. Its 288 tuples are correlated measurements from two training seeds, not 288 independent replications. Of the paired probes, 235 have the same winner in every held-out block and 53 show block heterogeneity.

\paragraph{Limits of the evidence.}
There are 24 sign wins, 264 spectral wins and no measured ties at the descriptive $10^{-5}$-nat band. This band has not been qualified against full-grid numerical repeatability, and actual LM loss near-ties remain untested. The historically used finite holdout lacks document identifiers; an untouched, document-independent final test is not established. The perfect development ranking of $Q_{\mathrm{actual}}$ requires frozen-predictor confirmation on fresh seeds and independent evaluation data. Separate quadratic software controls span both winners and ties but do not supply missing LM evidence. These results establish limitations of scheduled-update ranking in this setting; they do not establish an end-to-end optimizer ranking or invalidate the theorem's conservative-step guarantee.

\paragraph{Outstanding validation.}
The remaining experiments include independent calibration and numerical tie qualification; known-constant bound-certification checks; qualified Frobenius pairwise ordering; genuine stateful optimizer and equal-budget AdamW/Muon comparisons; joint embedding/head updates; and sensitivity, precision and cost measurements. The separately reported four-static-route comparison and warm-up/replay results are summarized in Appendix~\ref{app:assignment_followup}; they should not be described as unrun. The nine-arm implementation prepared for this development campaign has passed software checks but has not supplied those training results. Full per-seed/configuration reconciliation, one-shot and prespecified schedule controls, and independent testing of a frozen warm-up rule remain necessary. Score, realised geometry, switch time and lock state should be reported together. Broader activation, scale and correctly tied VLM claims need matched confirmation; reversible switching remains an extension.

\clearpage
\section{Static assignments, early locking and stage-dependent preferences}
\label{app:assignment_followup}

\paragraph{Study and evidence scope.}
A follow-up trains a $31$M-parameter GPT-2 on OpenWebText for $2{,}000$ updates under all four fixed sign/spectral assignments to embedding and output layers and the adaptive gate, with five seeds per setting. Paired checkpoint diagnostics separately compare the two candidate updates by their held-out loss decrease on three seeds. Table~\ref{tab:assignment_followup} transcribes the rounded final-loss labels from the follow-up summary figure. These are a separate experimental series from Table~\ref{tab:gscion} and Appendix~\ref{app:scheduled_calibration}; their values must not be pooled across configurations.

\paragraph{Local agreement and later mismatch.}
On the output head, every reported early recommendation for sign agrees with the larger held-out decrease across the three diagnostic seeds. This is evidence for those sampled decisions, not an error-rate estimate over all stages, layers or noise levels: the summary does not give the number of checkpoint comparisons. The gate chooses sign at step $150$ and remains locked through step $2{,}000$, whereas the paired checkpoint measurements report a spectral advantage around step $700$. Thus the assignment is already fixed for roughly $550$ updates when the preferred candidate changes. The summary figure distinguishes the post-switch score and its running average from a no-switch counterfactual trace. Their checkpoint correspondence is not specified in the summary; their crossings should not be equated with the held-out-loss crossover.

\paragraph{Assignment quality versus online adaptation.}
The rounded gap between the original gate and PP is $5.45-4.92=0.53$; PP is reported to outperform the gate on all five seeds. The warm-up variant improves the displayed gate loss to $5.23$, and a frozen-route replay also displays $5.23$. This equality concerns the warm-up variant, not the original $5.45$ gate, and rounded endpoints alone do not prove identical trajectories. Allowing reversal leaves the displayed final loss at $5.45$, and no return switch is reported. Inertia in the smoothed decision statistic is a plausible explanation, but the separate causal effects of the averaging horizon, lock and warm-up have not been isolated. These results support stage-dependent candidate preferences and diagnose limitations of the routing rule; they do not establish that ongoing measurement beats a static assignment, that the score is universally calibrated, or that the conservative-step bounds are empirically certified.

\paragraph{Separating the two evaluation questions.}
The checkpoint comparison tests whether the score predicts the better candidate update at the current state; the assignment matrix tests whether acting on the signal improves a complete training run. Agreement in the former does not imply improvement in the latter. The $Q_{\mathrm{actual}}$ result in Appendix~\ref{app:scheduled_calibration} comes from a separate development protocol and a different predictor, so it neither explains the follow-up's misses nor establishes that step size causes them. Likewise, the replay's rounded endpoint does not establish that the routing rule alone causes the gate's gap to the static spectral baseline. A diagnosis requires linking checkpoint calibration, raw and smoothed scores, realised routes and matched controller ablations.

\paragraph{Warm-up without hindsight.}
Choosing step $700$ because the same runs reveal a crossover would use evaluation outcomes to select the policy. A follow-up should instead choose a warm-up duration or stability rule using designated selection runs, freeze it, and evaluate on fresh seeds with all four static assignments, early one-shot selection and a fixed schedule. Equal tuning budgets and the same training protocol are needed. Raw score, smoothed score, realised geometry, switch times and lock state should be logged together. Reversal is a separate extension, with its averaging rule and thresholds frozen before evaluation.

\paragraph{Reproducibility boundary.}
The summary establishes the model/data family, training horizon, main five-seed comparison and three-seed diagnostic scope. It does not establish exact seed identities, all optimizer and schedule settings, candidate step scaling, checkpoint/holdout sampling, the warm-up rule, the replayed policy, or seed counts for every additional arm. Per-seed results, configuration hashes, route traces and the error-bar definition remain to be linked to the reported figure before treating this series as a fully reproducible benchmark. The checkpoint evidence is therefore retained at its reported scope, and is not used to override the broader scheduled-update development errors in Appendix~\ref{app:scheduled_calibration}.

\section{Additional details for the theoretical comparison}
\label{app:additional_theory}

\paragraph{Proof of Proposition~\ref{prop:lower_bound_sandwich}.}
\begin{proof}
\UPDATE{The two-sided Taylor remainder implied by $L_g$-smoothness gives}
\begin{equation}
\UPDATE{\eta\mu_g-\frac{L_g\eta^2}{2}
\le \Delta_g(\eta)
\le \eta\mu_g+\frac{L_g\eta^2}{2},
\qquad \mu_g:=\E[\inner{G}{D_g}\mid W].}
\end{equation}
\UPDATE{The alignment bound and H\"older's inequality give $S_g\le\mu_g\le M_g$. Substituting $\eta_g=S_g/L_g$ yields $B_g\le\Delta_g\le S_gM_g/L_g+S_g^2/(2L_g)=c_gB_g$; taking ratios proves Eq.~\eqref{eq:ratio_sandwich}.}
\end{proof}

\paragraph{Activation moment assumptions.}
\begin{definition}[Activation noise-to-signal ratio]\label{def:nsr}
For a pointwise activation $\sigma$ and $\gamma\sim\mathcal{N}(0,1)$ with $\E[\sigma(\gamma)]\neq 0$,
\begin{equation}\label{eq:nsr_def}
    \ssigma \;:=\; \frac{\E[\sigma(\gamma)^2]}{(\E[\sigma(\gamma)])^2}.
\end{equation}
\end{definition}

  \begin{proposition}[\UPDATE{Homogeneous activation specialization of {\citealp[Cor.~2.7]{DavisDrusvyatskiy2025SpecGD}}}]
 \label{prop:strank_activation}
  \UPDATE{Let $X$ have nonzero columns, let $W$ have iid Gaussian entries with the scaling of the cited result, and let $\sigma$ satisfy its homogeneity, growth, and nonzero-mean assumptions. Then, with the probability and finite-width factors stated in the cited corollary, $\strank{\sigma(WX)}$ is bounded by the corresponding $(1\pm\varepsilon)$ factor times $\ssigma$.
} \end{proposition}

  \UPDATE{For homogeneous ReLU and } ReLU $^2$\UPDATE{, the unit-scale ratios are $\pi$ and $6$, respectively. SiLU and GELU are non-homogeneous and require the scale-uniform form of the full theorem; their unit-scale values ($\approx8.33$ and $\approx5.34$) are descriptive statistics, not constants supplied by the homogeneous corollary. Appendix~\ref{app:activation_theory} states the scope in detail}.

\paragraph{Frobenius comparison.}
For $E=\widehat G-G$, take $D_\F=\widehat G/\norm{\widehat G}_\F$ when $\widehat G\ne0$, and $D_\F=0$ otherwise. Cauchy--Schwarz and the triangle inequality give
\[
\inner{G}{D_\F}\ge\norm{G}_\F-2\norm{E}_\F.
\]
With $\nu_\F=\E[\norm{E}_\F\mid W]$ and Frobenius smoothness $L_\F>0$, the conservative step yields $B_\F=(\norm{G}_\F-2\nu_\F)_+^2/(2L_\F)$. In the positive-signal regime, the same proof gives pairwise bounds involving $B_\F$. This extends the analytical comparison without requiring a three-way selector; it does not show that Frobenius updates are empirically dominated.

\end{document}